\documentclass{article} 
\usepackage[accepted]{icml2019}  
\usepackage{times}  
\usepackage{helvet}  
\usepackage{courier}  
\usepackage{url}  
\usepackage{graphicx}  
\usepackage{amsmath}
\usepackage{color}
\usepackage{booktabs}
\usepackage{microtype}
\usepackage{amsthm}
\usepackage{amsfonts}
\usepackage{amsmath,mathtools}
\usepackage{empheq}

\usepackage{amssymb}
\usepackage{hyperref}
\usepackage{nicefrac}
\usepackage{macros}
\usepackage{subcaption}
\usepackage{mathtools}

\usepackage[colorinlistoftodos,prependcaption]{todonotes}

\newcommand{\p}{m} 
\newcommand{\M}{\beta} 
\newcommand{\V}{\sigma^2} 
\newcommand{\lr}{\eta} 
\newcommand{\W}{\mathbf{W}} 
\newcommand{\F}{F} 
\newcommand{\sg}{g} 
\newcommand{\tg}{\nabla F} 
\newcommand{\lip}{L} 
\newcommand{\J}{\mathbf{J}} 
\newcommand{\I}{\mathbf{I}}
\newcommand{\spgap}{\zeta} 
\newcommand{\x}{\mathbf{x}} 
\newcommand{\avgx}{\overline{\mathbf{x}}} 
\newcommand{\X}{\mathbf{X}} 
\newcommand{\G}{\mathbf{G}} 
\newcommand{\cp}{\tau} 
\newcommand{\summp}{\sum_{i=1}^\p} 
\newcommand{\z}{\mathbf{z}} 
\newcommand{\crounds}{K/\tau}
\newcommand{\matA}{\mathbf{A}}
\newcommand{\matB}{\mathbf{B}}
\newcommand{\matPhi}{\mathbf{\Phi}}

\newcommand{\gtg}{\mathcal{H}} 
\newcommand{\gsg}{\mathcal{G}} 
\newcommand{\genx}{\mathbf{u}} 
\newcommand{\genlr}{\lr_{\text{eff}}} 
\newcommand{\msg}{\mathbf{Y}} 
\newcommand{\mtg}{\mathbf{Q}}
\newcommand{\CExs}{\mathbf{E}_{k}}
\newcommand{\N}{N}
\newcommand{\minib}{\xi}

\newcommand{\alg}{\mathcal{A}}
\newcommand{\addw}{v}

\usepackage{cleveref}
\crefname{equation}{}{}
\Crefname{equation}{}{}
\crefname{thm}{theorem}{theorems}
\Crefname{thm}{Theorem}{Theorems}
\crefname{clm}{claim}{claims}
\Crefname{clm}{Claim}{Claims}
\Crefname{coro}{Corollary}{Corollaries}
\Crefname{lem}{Lemma}{Lemmas}
\Crefname{sec}{Section}{Sections}
\crefname{app}{appendix}{appendices}
\Crefname{app}{Appendix}{Appendices}
\Crefname{part}{Part}{Parts}
\crefname{prop}{proposition}{propositions}
\Crefname{prop}{Proposition}{Propositions}
\Crefname{propty}{Property}{Properties}
\crefname{figure}{fig.}{figures}
\Crefname{figure}{Figure}{Figures}
\crefname{defn}{definition}{definitions}
\Crefname{defn}{Definition}{Definitions}
\crefname{fact}{fact}{facts}
\Crefname{fact}{Fact}{Facts}
\crefname{appendix}{appendix}{appendices}
\Crefname{appendix}{Appendix}{Appendices}
\crefname{algo}{algorithm}{algorithms}
\Crefname{algo}{Algorithm}{Algorithms}
\crefname{algorithm}{algorithm}{algorithms}
\Crefname{algorithm}{Algorithm}{Algorithms}
\crefname{conj}{conjecture}{conjectures}
\Crefname{conj}{Conjecture}{Conjectures}
\crefname{obs}{observation}{observations}
\Crefname{obs}{Observation}{Observations}
\crefname{assump}{assumption}{assumptions}
\Crefname{assump}{Assumption}{Assumptions}
\crefname{rem}{remark}{remarks}
\Crefname{rem}{Remark}{Remarks}

\begin{document}
%

\twocolumn[
\icmltitle{Cooperative SGD: A Unified Framework for the Design and Analysis of Communication-Efficient SGD Algorithms}



\icmlsetsymbol{equal}{*}

\begin{icmlauthorlist}
\icmlauthor{Jianyu Wang}{to}
\icmlauthor{Gauri Joshi}{to}
\end{icmlauthorlist}

\icmlaffiliation{to}{Department of Electrical \& Computer Engineering, Carnegie Mellon University, Pittsburgh, PA, USA}

\icmlcorrespondingauthor{Jianyu Wang}{jianyuw1@andrew.cmu.edu}
\icmlcorrespondingauthor{Gauri Joshi}{gaurij@andrew.cmu.edu}


\vskip 0.3in
]



\printAffiliationsAndNotice{}  

\begin{abstract}
Communication-efficient SGD algorithms, which allow nodes to perform local updates and periodically synchronize local models, are highly effective in improving the speed and scalability of distributed SGD. However, a rigorous convergence analysis and comparative study of different communication-reduction strategies remains a largely open problem. This paper presents a unified framework called Cooperative SGD that subsumes existing communication-efficient SGD algorithms such as periodic-averaging, elastic-averaging and decentralized SGD. By analyzing Cooperative SGD, we provide novel convergence guarantees for existing algorithms. Moreover, this framework enables us to design new communication-efficient SGD algorithms that strike the best balance between reducing communication overhead and achieving fast error convergence with low error floor.
\end{abstract}

\section{Introduction}
\label{sec:intro}

Stochastic gradient descent (SGD) is the core backbone of most state-of-the-art machine learning algorithms. Due to its widespread applicability, speeding-up SGD is arguably the single most impactful and transformative problem in machine learning. Classical SGD was designed to be run on a single computing node, and its error-convergence has been extensively analyzed and improved in optimization and learning theory \cite{dekel2012optimal,ghadimi2013stochastic}. Due to the massive training data-sets and deep neural network architectures used today, running SGD at a single node can be prohibitively slow. This calls for distributed implementations of SGD, where gradient computation and aggregation is parallelized across multiple worker nodes. Although parallelism boosts the amount of data processed per iteration, it exposes SGD to unpredictable synchronization and communication delays stemming from variability in the computing infrastructure. This work presents a unified framework called Cooperative SGD to analyze communication-efficient distributed SGD algorithms that periodically average models trained locally at different computing nodes.

\textbf{Limitations of Parameter Server Framework.}
A commonly used method to parallelize gradient computation and process more training data per iteration is the parameter server framework \citep{dean2012large,li2014scaling,cui2014exploiting}. Each of the $\p$ worker nodes computes the gradients of one mini-batch of data, and a parameter server aggregates these gradients and updates the model parameters. Synchronization delays in waiting for slow workers can be alleviated via asynchronous gradient aggregation \citep{recht2011hogwild, cui2014exploiting, gupta2016model, mitliagkas2016asynchrony, dutta2018slow}. However it is difficult to eliminate communication delays since by design, parameter server framework requires gradients and model updates to be communicated between the parameter server and workers after every iteration. 

\textbf{Communication-Efficient SGD.}
To address the limitations of the parameter server framework, recent works proposed communication-efficient SGD variants that perform more computation at worker nodes. A natural idea is to allow workers to \textit{perform $\tau$ local updates to the model instead of just computing gradients, and then periodically averaging the local models} \citep{moritz2015sparknet,zhang2016parallel,povey2014parallel,su2015experiments,chaudhari2017parle,smith2018cocoa,lin2018don}. A similar approach (averaging after several epochs) is referred to \emph{Federated Averaging} (FedAvg) \citep{mcmahan2016communication} in recent works and it is shown to work well even for non-i.i.d. data partitions. Although extensive empirical results have validated the effectiveness of periodic averaging, rigorous theoretical understanding of how its convergence depends on the number of local updates $\tau$ is quite limited \citep{zhou2017convergence,yu2018parallel,stich2018local}.

Instead of simply averaging the local models every $\tau$ iterations, Elastic-averaging SGD (EASGD) proposed in \cite{zhang2015deep} \textit{adds a proximal term to the objective function in order to allow some slack between the models} -- an idea that is drawn from the Alternating Direction Method of Multipliers (ADMM)  \citep{boyd2011distributed,parikh2014proximal}. 
Although the efficiency of EASGD and its asynchronous and periodic averaging variants has been empirically validated \citep{zhang2015deep,chaudhari2017parle}, its convergence analysis under general convex or non-convex objectives is an open problem. The original paper \citep{zhang2015deep} only gives an analysis of vanilla EASGD for quadratic objective functions.

A different approach to reducing communication is to perform \emph{decentralized training with sparse-connected network of worker nodes}. Each node only synchronizes with its neighbors, thus reducing the communication overhead significantly. Decentralized averaging has a long history in the distributed and consensus optimization community \citep{tsitsiklis1986distributed,nedic2009distributed,duchi2012dual,tsianos2012communication,zeng2016nonconvex,yuan2016convergence,sirb2018decentralized,bijral2017data}. Most of these works are for gradient descent or dual averaging methods rather than stochastic gradient descent (SGD), and they do not allow workers to make local updates. Recently, decentralized averaging was successfully applied to deep learning in \citep{jin2016scale,jiang2017collaborative,lian2017can}, which also provide convergence analyses for $1$ local update per worker. It is still unclear how decentralized training compares with periodic averaging ($\tau$ updates per worker).

\textbf{Main Contributions.}
A common thread in all the communication-efficient SGD methods described above is that \textit{they allow worker nodes to perform local model-updates and limit the synchronization/consensus between the local models}. Limiting model-synchronization reduces communication overhead, but it increases model discrepancies and can give an inferior error convergence performance. Communication-efficient SGD algorithms seek to strike the best trade-off between error-convergence and communication-efficiency.

In this paper, we propose a powerful framework called \emph{Cooperative SGD} that enables us to obtain an integrated analysis and comparison of communication-efficient algorithms. Existing algorithms including periodic averaging SGD, Elastic Averaging SGD, decentralized SGD are special cases of cooperative SGD, and thus can be analyzed under one single umbrella. The main contributions of this paper are:
\begin{enumerate}
    \item We present the first unified convergence analysis for the cooperative SGD class (\Cref{sec:convergence}) of algorithms that subsumes periodic, elastic and decentralized averaging. The theoretical results reveal how different communication-efficient strategies influence the error-convergence performance. 
    \item In particular, we provide the first analysis of elastic-averaging SGD for non-convex objectives, and use it to determine the best elasticity parameter $\alpha$ (\Cref{sec:easgd}) that achieves the lowest error floor at convergence. 
    \item We obtain a new analysis and tighter error bound for periodic averaging SGD by removing the uniformly bounded gradients assumption (\Cref{sec:main_thm}). The analysis can be applied to FedAvg with i.i.d. data partitions as well.
    \item Based on the unified analysis, we show the first in-depth comparison between periodic/elastic-averaging with decentralized training methods and design new communication-efficient SGD variants by combining existing strategies (see \Cref{sec:new}).
\end{enumerate}

An alternative approach to communication-efficiency is gradient compression techniques \citep{wangni2017gradient,wen2017terngrad,lin2017deep} that quantize the gradients computed by workers. Although interesting and important, this approach is beyond the scope of our paper; we focus on communication-efficiency via local updates at workers.

\section{Preliminaries}
\label{sec:preliminaries}

In this section we present the update rules of existing communication-efficient SGD algorithms in terms of our notation that is used in the rest of the paper.

\textbf{Notation.} 
All vectors considered in this paper are column vectors. For convenience, we use $\one$ to denote $[1,1,\dots,1]\tp$ and define matrix $\J = \one\one\tp/(\one\tp \one)$. Unless otherwise stated, $\one$ is a size $\p$ column vector, and the matrix $\J$ and identity matrix $\I$ are of size $\p \times \p$, where $\p$ is the number of workers. Let $\vecnorm{\cdot}$, $\fronorm{\cdot}$ and $\opnorm{\cdot}$ denote the $\ell_2$ vector norm, Frobenius matrix norm and operator norm, respectively.

\textbf{Fully Synchronous SGD.}
Suppose the model parameters are denoted by $\x \in \mathbb{R}^d$ and the training set is denoted by $\mathcal{S} = \{s_1, \dots, s_\N\}$, where $s_i$ represents the $i$-th data sample. Then, the interested problem is the minimization of the empirical risk as follows:
\begin{align}
    \min_{\x \in \mathbb{R}^d} \brackets{ \F(\x) := \frac{1}{\N}\sum_{i=1}^\N f(\x; s_i)} \label{eqn:min}
\end{align}
where $f(\cdot)$ is the loss function defined by the learning model. In the distributed setting, there are total $\p$ worker machines that compute stochastic gradients in parallel. The updates can be written as
\begin{align}
    \x_{k+1} = \x_k - \lr\brackets{\frac{1}{\p}\summp \sg(\x_k; \xi_k^{(i)})}
\end{align}
where $\lr$ is the learning rate, $\xi_k^{(i)}\subset \mathcal{S}$ are randomly sampled mini-batches, and $g(\x;\xi) = \frac{1}{|\xi|}\sum_{s_i \in \xi}\nabla f(\x;s_i)$ denotes the stochastic gradient. For simplicity, we will use $g(\x)$ instead of $g(\x; \xi)$ in the rest of the paper. 

\textbf{Periodic Averaging SGD (PASGD).}
Local models are averaged after every $\cp$ iterations. Its update rule is
\begin{align}
    \x_{k+1}^{(i)} = 
    \begin{cases}
    \frac{1}{\p}\sum_{j=1}^\p \brackets{\x_k^{(j)} - \lr \sg(\x_k^{(j)})}, & k \modd \cp = 0 \\
    \x_k^{(i)} - \lr \sg(\x_k^{(i)}), & \text{otherwise}
    \end{cases}
\end{align}
where $\x_k^{(i)}$ denotes the model parameters in the $i$-th worker and $\cp$ is defined as the communication period. The recently proposed federated learning framework \citep{mcmahan2016communication}  also performs periodic averaging, but with non i.i.d.\ local datasets.

\textbf{Elastic Averaging SGD (EASGD).}
Instead of performing a simple average of the local models, the elastic-averaging algorithm (EASGD) proposed in \cite{zhang2015deep} maintains an auxiliary variable $\z_{k}$. This variable is used as an anchor while updating the local models $\x_{k}^{(i)}$. The update rule of vanilla EASGD\footnote{The paper \citep{zhang2015deep} also presents periodic averaging and momentum variants of EASGD. However, only vanilla EASGD has been theoretically analyzed, and only for quadratic loss functions.} is given by
\begin{align}
    \x_{k+1}^{(i)} &= \x_{k}^{(i)} - \lr \sg(\x_k^{(i)}) - \alpha (\x_k^{(i)} - \z_k), \label{eqn:easgd_o1}\\
    \z_{k+1} &= (1-\p\alpha)\z_k + \p\alpha\avgx_k, \label{eqn:easgd_o2}
\end{align}
where $\avgx_k = \summp \x_k^{(i)}/\p$. A larger value of the parameter $\alpha$ forces more consensus between the locally trained models and improves stability, but it may reduce the convergence speed -- a phenomenon that is not yet well-understood.

\textbf{Decentralized SGD (D-PSGD).}
The decentralized SGD algorithm D-PSGD (also referred as consensus-based distributed SGD), was proposed by \citep{jiang2017collaborative,lian2017can}. Nodes perform local updates and average their models with neighboring nodes, where the network topology is captured by a mixing matrix $\W$. The update rule is
\begin{align}
    \x_{k+1}^{(i)} = \sum_{j=1}^{\p} w_{ji}\x_k^{(j)} - \lr \sg(\x_k^{(i)}) \label{eqn:decen_sgd}
\end{align}
where $w_{ji}$ is the $(j,i)^{th}$ element of the mixing matrix $\W$, and it represents the contribution of node $j$ in the averaged model at node $i$.

\section{The Cooperative SGD Framework}
\subsection{Key Elements and Update Rule}
The Cooperative SGD algorithm is denoted by $\alg(\cp,\W,\addw)$, where $\cp$ is the number of local updates, $\W$ is the mixing matrix used for model averaging, and $\addw$ is the number of auxiliary variables. These parameters feature in the update rule as follows.
\begin{enumerate}
    \item \textbf{Model Versions at Workers.} At iteration $k$, the $\p$ workers have different versions $\x_k^{(1)},\dots,\x_k^{(\p)} \in \mathbb{R}^d$ of the model. In addition, there are $\addw$ auxiliary variables $\z_k^{(1)}, \dots, \z_k^{(\addw)}$ that are either stored at $\addw$ additional nodes or at one or more of the workers, depending upon implementation. 
    \item \textbf{Gradients and Local Updates.} In each iteration, the workers evaluate the gradient $\sg(\x_k^{(i)})$ for one mini-batch of data and update $\x_k^{(i)}$. The auxiliary variables are only updated by averaging a subset of the local models as described in point $3$ below. Thus, their gradients are zero, i.e., $\sg(\z_k^{(j)}) = \mathbf{0}, \forall j \in \{1,\dots,\addw\}, \forall k$.
    \item \textbf{Model-Averaging.} In iteration $k$, the local models and auxiliary variables are averaged with neighbors according to mixing matrix $\W_k \in \mathbb{R}^{(\p+\addw)\times(\p+\addw)}$. To capture periodic averaging, we use a time-varying $\W_k$ that varies as:
    \begin{align}
    \W_k = 
    \begin{cases}
    \W, & k \modd \cp = 0 \\
    \mathbf{I}_{(\p+\addw)\times(\p+\addw)}, & \text{otherwise},
    \end{cases}
    \label{eqn:w_k}
    \end{align}
    where the identity mixing matrix $\mathbf{I}_{(\p+\addw)\times(\p+\addw)}$ means that there is no inter-node communication during the $\tau$ local updates.
\end{enumerate}

We now present a general update rule that combines the above elements. Define matrices $\X_k, \G_k \in \mathbb{R}^{d \times (\p+\addw)}$ that concatenate all local models and gradients:
\begin{align}
    \X_k =& [\x_k^{(1)},\dots,\x_k^{(\p)}, \z_k^{(1)}, \dots, \z_k^{(\addw)}], \\
    \G_k =& [\sg(\x_k^{(1)}),\dots,\sg(\x_k^{(\p)}), \mathbf{0}, \dots, \mathbf{0}].
\end{align}
The update rule in terms of these matrices is
\begin{align}
    \X_{k+1} = \parenth{\X_k - \lr \G_k}\W_k.
    \label{eqn:mat_update}
\end{align}

\begin{rem}\label{rem:alter}
Instead of using update \Cref{eqn:mat_update}, one can use an alternative rule: $\X_{k+1} = \X_k\W_k - \lr\G_k$. The convergence analyses and insights in this paper can be extended to this update rule. We choose to study the update rule \Cref{eqn:mat_update} for all existing algorithms (PASGD, EASGD, D-PSGD) since fully synchronous SGD corresponds to the special case $\W_k = \J$.
\end{rem}

\subsection{Existing Algorithms as Special Cases} \label{sec:cases}
We now show how existing communication-efficient algorithms are special cases of the general Cooperative SGD framework $\alg(\cp,\W,\addw)$.

\textbf{Fully synchronous SGD $\Leftrightarrow \alg(1, \J, 0)$.} 
The local models are synchronized with all other workers after every iteration. 

\textbf{PASGD $ \Leftrightarrow \alg(\cp, \J, 0)$.} 
The local models are synchronized with all other workers after every $\cp$ iterations. 

\textbf{EASGD $\Leftrightarrow \alg(1, \W_\alpha, 1)$.} 
In EASGD, there is one auxiliary variable. Besides, the mixing matrix is controlled by a hyper-parameter $\alpha$ as follows
\begin{align}
    \W_\alpha =& 
    \begin{bmatrix}
    (1-\alpha)\I & \alpha \one \\
    \alpha\one\tp & 1-\p\alpha
    \end{bmatrix} \in \mathbb{R}^{(\p+1)\times(\p+1)}.\label{eqn:easgd_w}
\end{align}
One can easily validate that the updates defined in \Cref{eqn:mat_update,eqn:w_k,eqn:easgd_w} are equivalent to \Cref{eqn:easgd_o1,eqn:easgd_o2} when using the alternative update rule $\X_{k+1} = \X_k\W_k - \lr \G_k$.

\textbf{D-PSGD $\Leftrightarrow \alg(1, \W, 0)$.} 
The mixing matrix $\W$ in D-PSGD is fixed as a sparse matrix. Only one local update before averaging is considered and there are no auxiliary variables. In addition to these special cases, the cooperative SGD framework allows us to design other communication-efficient SGD variants, as we describe in \Cref{sec:new}.

\subsection{Communication Efficiency}
\label{sec:comm_eff}
The cooperative SGD framework improves the communication-efficiency of distributed SGD in three different ways, as described below. We illustrate these in \Cref{fig:comm_reduc}, which compares the execution timeline of cooperative SGD with fully synchronous SGD.

\textbf{Periodic Averaging.}
The communication delay is amortized over $\tau$ iterations and is $\tau$ times smaller than fully synchronous SGD. Moreover, periodic averaging evens out random variations in workers' computing time, and alleviates the synchronization delay in waiting for slow workers. Observe in \Cref{fig:comm_reduc} that the idle time of workers is significantly reduced.
\begin{figure}[t!]
    \centering
    \begin{subfigure}[t]{0.48\textwidth} 
        \centering
        \includegraphics[width=\columnwidth]{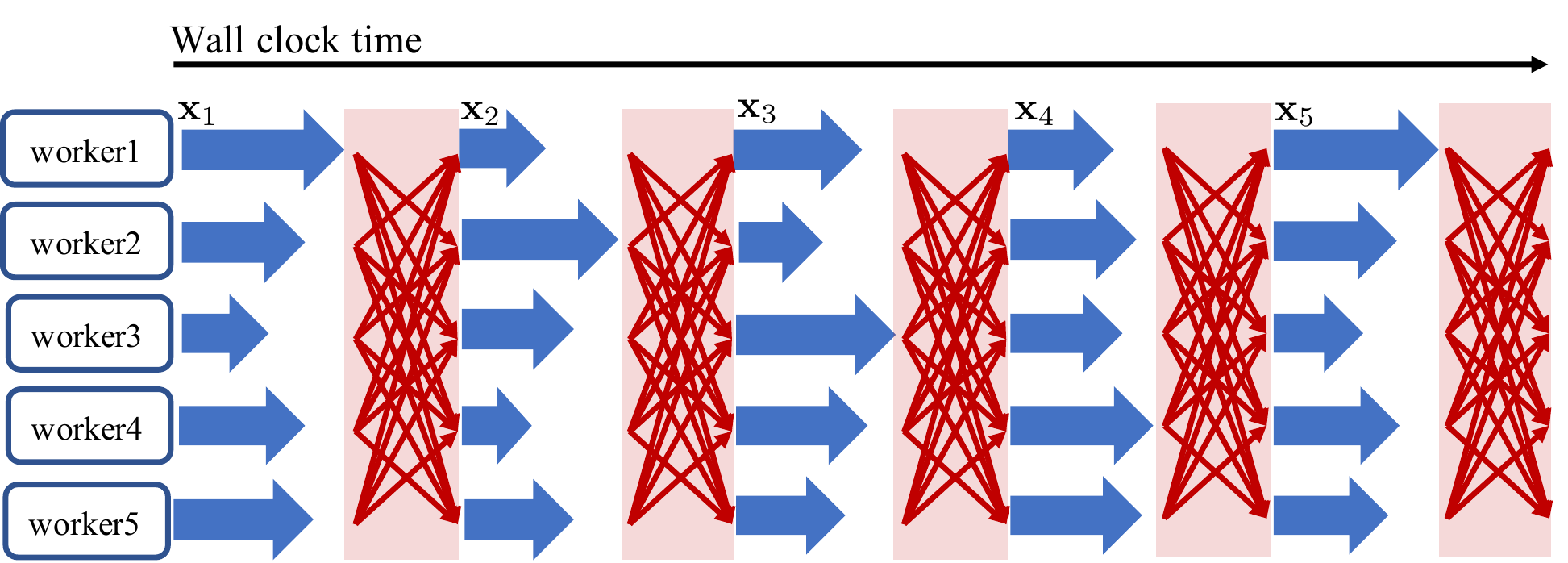}
        \caption{Fully synchronous SGD.}
    \end{subfigure}
    
    \begin{subfigure}[t]{0.48\textwidth}
        \centering
        \includegraphics[width=\columnwidth]{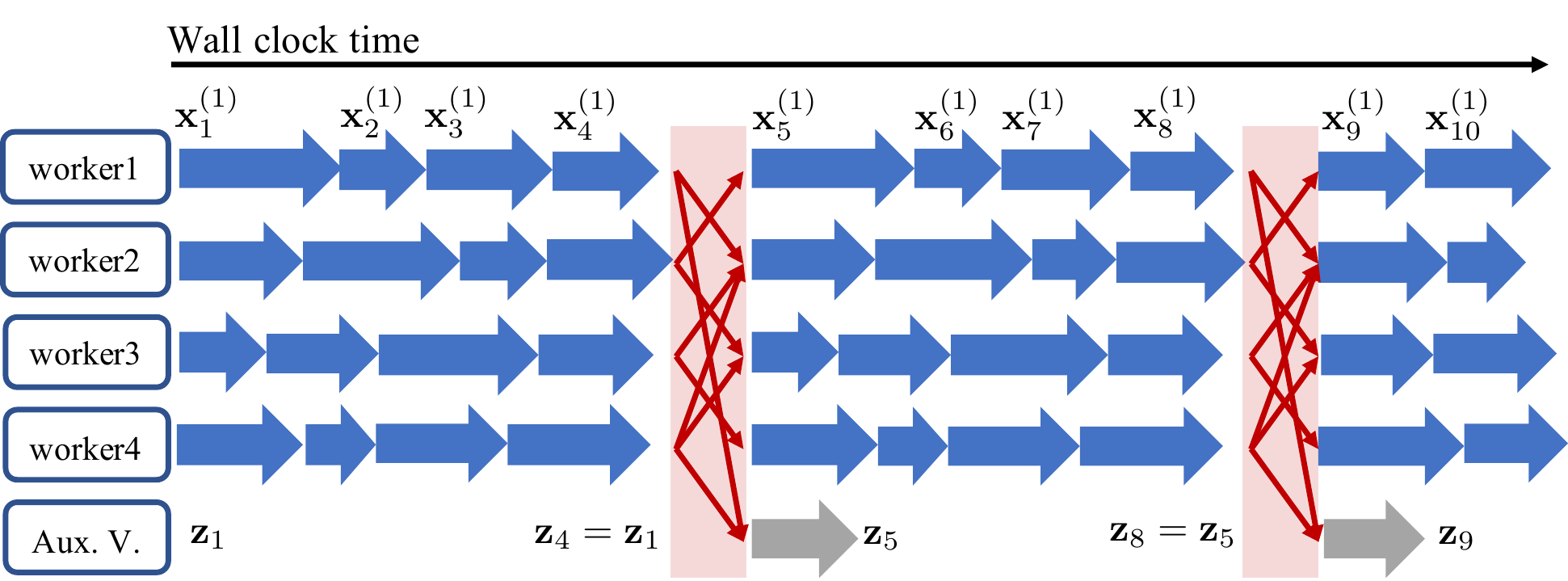}
        \caption{Cooperative SGD.}
    \end{subfigure}
    \caption{Illustration of communication-reduction strategies for $\cp = 4$. Blue, red, grey arrows represent gradient computation, communication among workers, and update of auxiliary variables respectively.}
    \label{fig:comm_reduc}
\end{figure}

\textbf{Non-blocking Execution.} 
Since the auxiliary variables do not compute gradients, they remain the same while worker nodes conduct local updates, that is, $\z_{j \cp} = \z_{j \cp -1} = \dots = \z_{(j-1)\cp+1}$ for $j \geq 1$. Thus, the worker nodes only need $\z_{(j-1)\tau+1}$ before the model-averaging step from $\x_{j \tau}$ to $\x_{j \tau+1}$. So, the auxiliary variables can perform and broadcast model-updates while the workers perform the next set of local updates (see \Cref{fig:comm_reduc}), thus reducing synchronization delay. 

\textbf{Group Synchronization.}
Lastly, instead of synchronizing with all workers, a local model just needs to exchange information with its neighbors, where the network topology is captured by the mixing matrix $\W$. Thus, using a sparse mixing matrix $\W$ reduces the overall communication delay incurred per iteration.

\section{Unified Convergence Analysis}
\label{sec:convergence}

In this section, we present the unified convergence analysis of algorithms in cooperative SGD framework and study how the $\cp$, $\W$, and $\addw$ affect the error-convergence.

\subsection{Assumptions}
The convergence analysis is conducted under the following assumptions, which are similar to previous works on the analysis of distributed SGD \citep{bottou2016optimization}:
\begin{enumerate}
    \item (Smoothness): $\vecnorm{\tg(\x)-\tg(\mathbf{y})} \leq \lip\vecnorm{\x-\mathbf{y}}$;
    \item (Lower bounded): $F(\x) \geq \F_\text{inf}$;
    \item (Unbiased gradients): $\Exs_{\xi|\x} \brackets{\sg(\x)} = \tg(\x)$;
    \item (Bounded variance): $\Exs_{\xi|\x} \vecnorm{\sg(\x)-\tg(\x)}^2 \leq \M\vecnorm{\tg(\x)}^2 + \V$
    where $\M$ and $\V$ are non-negative constants and in inverse proportion to the mini-batch size.
    \item (Mixing Matrix): $\W\one_{\p+\addw} = \one_{\p+\addw}, \ \W\tp = \W$. Besides, the magnitudes of all eigenvalues except the largest one are strictly less than $1$: $\max\{|\lambda_2(\W)|,|\lambda_{\p+\addw}(\W)|\} < \lambda_1(\W)=1$.
\end{enumerate}

\subsection{Update Rule for the Averaged Model}
To facilitate the convergence analysis, we firstly introduce the quantities of interests. Multiplying $\one_{\p+\addw}/(\p+\addw)$ on both sides in \eqref{eqn:mat_update}, we get
\begin{align}
    \X_{k+1} \frac{\one_{\p+\addw}}{\p+\addw}
    =& \X_k \frac{\one_{\p+\addw}}{\p+\addw} - \lr \G_k \frac{\one_{\p+\addw}}{\p+\addw}
\end{align}
where $\W_k$ disappears due to the special property from Assumption 5: $\W_k \one_{\p+\addw} = \one_{\p+\addw}$. Then, define the average model and effective learning rate as
\begin{align}
    \genx_k =  \X_k \frac{\one_{\p+\addw}}{\p+\addw}, \ \genlr = \frac{\p}{\p+\addw}\lr. \label{eqn:def_gen}
\end{align}
After rearranging, one can obtain
\begin{align}
    \genx_{k+1} = \genx_k - \genlr\brackets{\frac{1}{\p}\summp \sg(\x_{k}^{(i)})} \label{eqn:gen_upd}
\end{align}
Observe that the averaged model $\genx_k$ is performing perturbed stochastic gradient descent. In the sequel, we will focus on the convergence of the averaged model $\genx_k$, which is common practice in distributed optimization literature \citep{nedic2009distributed,duchi2012dual,yuan2016convergence}. 

Since the objective function $\F(\x)$ is non-convex, SGD may converge to a local minimum or saddle point. Thus, the expected gradient norm is used as an indicator of convergence \citep{lian2015asynchronous,zeng2016nonconvex,bottou2016optimization}. We say the algorithm achieves an $\epsilon$-suboptimal solution if:
\begin{align}
    \Exs\brackets{\frac{1}{K}\sum_{k=1}^K\vecnorm{\tg(\genx_k)}^2} \leq \epsilon.
\end{align}
This condition guarantees convergence of the algorithm to a stationary point.

\subsection{Main Results}\label{sec:main_thm}
In deep learning, it is common to keep the learning rate as a constant and decay it only when the training procedure saturates. Thus, we present the analysis for fixed learning rate case and study the error floor at convergence.

\begin{thm} [\textbf{Convergence of Cooperative SGD}]
    For algorithm $\alg(\cp,\W, \addw)$, suppose the total number of iterations $K$ can be divided by the communication period $\cp$. Under Assumptions 1--5 (with $\M =0$ \footnote{Constant $\M$ in Assumption 4 only influences the constraint on the learning rate \eqref{eqn:lr_con} and will not appear in the expression of gradient norm upper bound \eqref{eqn:err_bnd}. In order to get neater results, $\M$ is set as 0 in the main paper. In the Appendix, we provide the proof for arbitrary $\M$.}), if the learning rate satisfies
    \begin{align}
        \genlr\lip + 5\genlr^2\lip^2\brackets{\parenth{1+\frac{\addw}{\p}}\frac{\cp}{1-\spgap}}^2 \leq 1 \label{eqn:lr_con}
    \end{align}
    where $\spgap = \max\{|\lambda_2(\W)|,|\lambda_{\p+\addw}(\W)|\}$, and all local models are initialized at a same point $\genx_1$, then the average-squared gradient norm after $K$ iterations is bounded as follows
    \begin{align}
        \Exs\brackets{\frac{1}{K}\sum_{k=1}^K\vecnorm{\tg(\genx_k)}^2} 
        \leq \underbrace{\frac{2\brackets{\F(\genx_1) - \F_\text{inf}}}{\genlr K} + \frac{\genlr\lip\V}{\p}}_{\text{fully sync SGD}} +\nonumber \\
             \underbrace{\genlr^2\lip^2\V\parenth{\frac{1+\spgap^2}{1-\spgap^2}\cp-1}\parenth{1+\frac{\addw}{\p}}^2}_{\text{network error}} \label{eqn:err_bnd} \\
        \xrightarrow{K\to\infty} \frac{\genlr\lip\V}{\p} + \genlr^2\lip^2\V\parenth{\frac{1+\spgap^2}{1-\spgap^2}\cp-1}\parenth{1+\frac{\addw}{\p}}^2 \label{eqn:err_floor}
    \end{align}
    where $\genx_k,\genlr$ are defined in \eqref{eqn:def_gen}.
    \label{thm:gen}
\end{thm}
All proofs are provided in the Appendix. The error floor at convergence is given by \eqref{eqn:err_floor}. 

\textbf{Error decomposition.}
It is worth noting that the upper bound \eqref{eqn:err_bnd} is decomposed into two parts. The first two terms are same as the optimization error bound in fully synchronous SGD \citep{bottou2016optimization}. The last term is \textit{network error}, resulted from performing local updates and reducing inter-worker communication. It directly increases the error floor at convergence and is a measure of local models' discrepancies. When all local models are fully synchronized at every iterations ($\cp = 1, \spgap = 0, \addw = 0$), then the network error becomes zero. 

\textbf{Dependence on $\cp,\W$.}
\Cref{thm:gen} states that the error floor at convergence \eqref{eqn:err_floor} is determined by the communication period $\cp$ and the second largest absolute eigenvalue $\spgap$ of the mixing matrix. In particular, the bound will monotonically increase along with $\tau$ and $\spgap$. The definition of $\spgap$ is common in random walks on graphs and reflects the mixing rates of different variables. When there is no communication among local workers, then $\W=\I_{\p+\addw}$ and $\spgap = 1$; When local models are fully synchronized, then $\W = \J_{\p+\addw}$ and $\spgap = 0$. Typically, a sparser matrix means a larger value of $\zeta$.

Besides, since the network error bound is linear to $\tau$ but proportional to $(1+\spgap^2)/(1-\spgap^2)$, as shown in \Cref{fig:ne}, it is more sensitive to the changes in communication period. 
%
In \Cref{fig:exp_ne}, we evaluate various hyper-parameter settings for training VGGNet \citep{simonyan2014very} for classification of the CIFAR10 dataset \citep{krizhevsky2009learning}. As suggested by \Cref{eqn:err_floor}, the empirical results show that a higher network error (larger $\cp$ or larger $\zeta$) leads to a higher error floor at convergence.

\textbf{Dependence on $\addw$.}
Note that the effective learning rate \eqref{eqn:def_gen} is determined by the number of auxiliary variables. Using more auxiliary variables results in smaller effective learning rate, since they update only through model averaging. Consequently, it may slow down the optimization progress (increase the first term in \eqref{eqn:err_bnd}) while enable smaller error floor at convergence (reduce the second term in \eqref{eqn:err_bnd}).
\begin{figure}[!t]
    \centering
    \includegraphics[width=.4\textwidth]{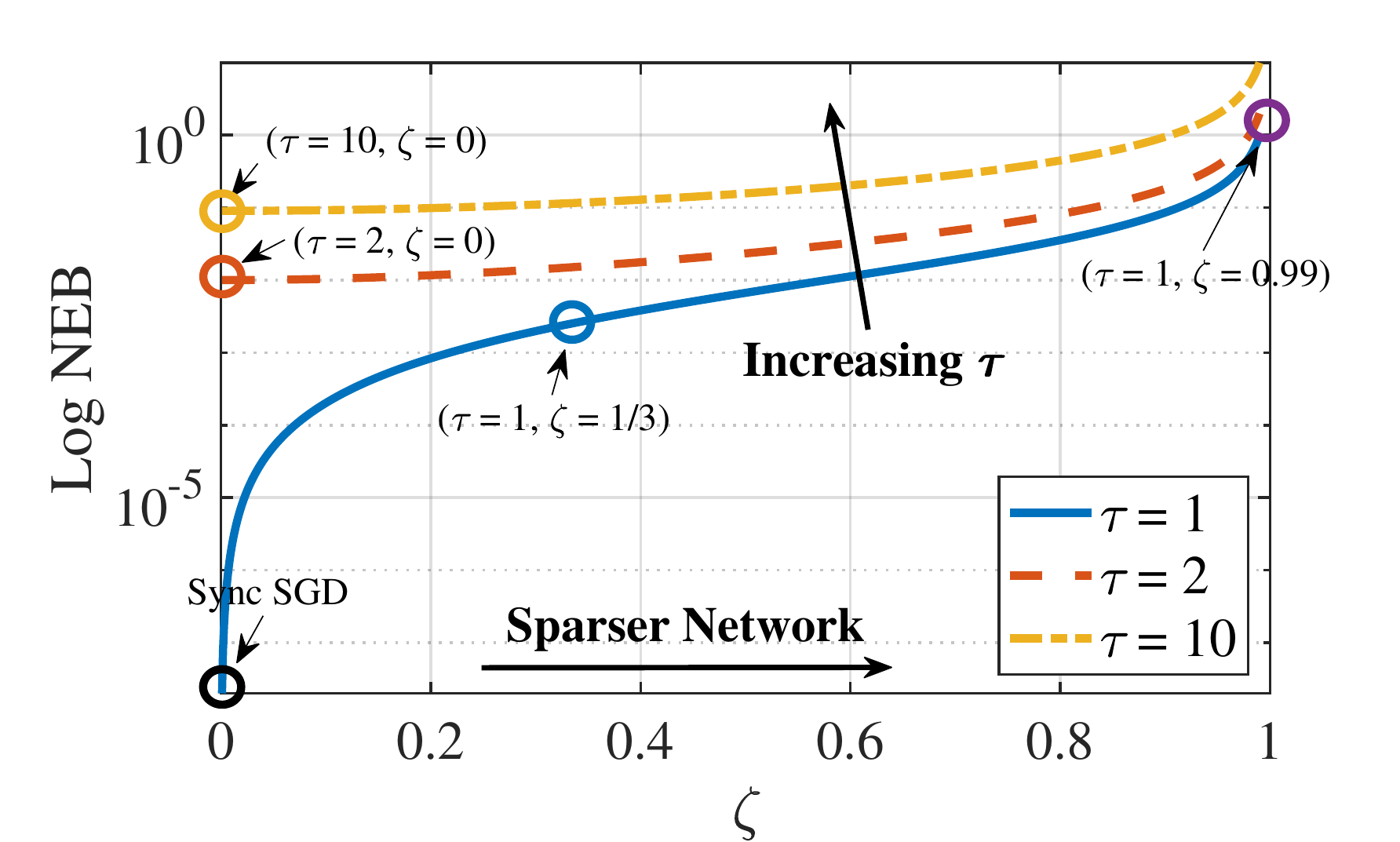}
    \caption{Illustration of how the network error bound in \Cref{eqn:err_bnd} monotonically increases with $\tau$ and $\spgap$.}
    \label{fig:ne}
\end{figure}

\begin{figure}[!t]
    \centering
    \includegraphics[width=.4\textwidth]{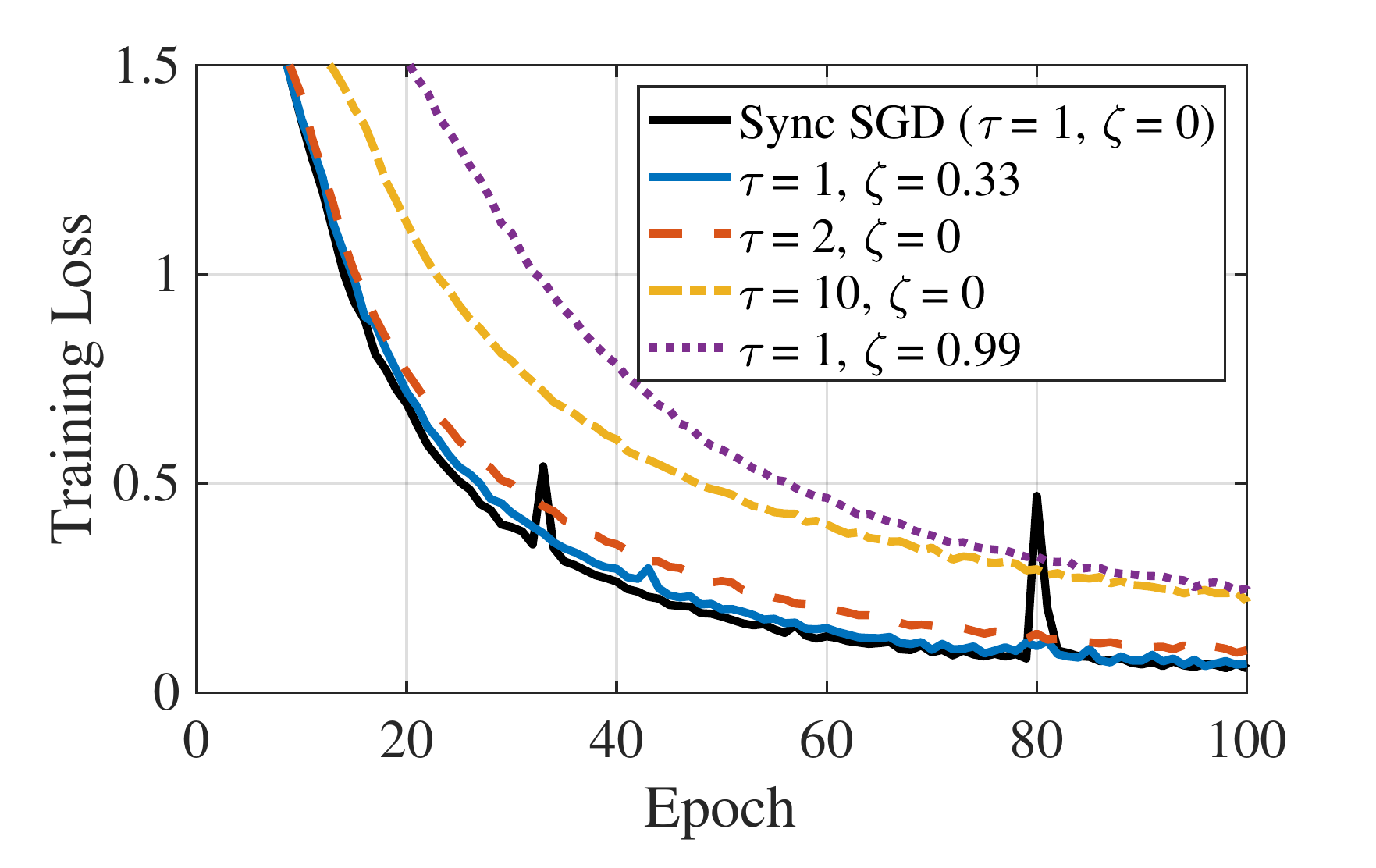}
    \caption{Experiments on CIFAR-10 with VGG-16 and 8 worker nodes. For the same learning rate, larger $\cp$ or larger $\zeta$ lead to a higher error floor at convergence. Each line corresponds to a circled point in \Cref{fig:ne}. 
    }
    \label{fig:exp_ne}
\end{figure}
\textbf{Finite horizon result.}
If $K$ is decided preemptively, then with a proper learning rate, we obtain the following bound. A similar approach also appears in \cite{ghadimi2013stochastic,lian2017can,yu2018parallel,bernstein2018signsgd}.
\begin{corollary}
    For algorithm $\alg(\cp,\W,\addw)$, under Assumption 1--5, if the learning rate is $\lr = \frac{\p+\addw}{\lip\p}\sqrt{\frac{\p}{K}}$, the average-squared gradient norm after $K$ iterations is bounded by
    \begin{align}
        \Exs\brackets{\frac{1}{K}\sum_{k=1}^K\vecnorm{\tg(\genx_k)}^2}& 
        \leq \frac{2\lip\brackets{\F(\genx_1) - \F_\text{inf}} + \V}{ \sqrt{\p K}} + \nonumber\\
        & \frac{\p}{K}\parenth{1+\frac{\addw}{\p}}^2\parenth{\frac{1+\spgap^2}{1-\spgap^2}\cp -1}\V
    \end{align}
    if the total iterations $K$ is sufficiently large: $K \geq 10\p[(1+\frac{\addw}{\p})\frac{\cp}{1-\spgap}]^2$. Furthermore, if $K \geq (\p+\addw)^2\p[(1+\frac{\addw}{\p})\frac{\cp}{1-\spgap}]^2$, then the average-squared gradient norm will be bounded by $2[\lip(\F(\x_1)-\F_\text{inf})+\V]/\sqrt{\p K}$.
    \label{corollary:finite_K}
\end{corollary}
By directly setting $\W = \J$ (i.e., $\spgap = 0$) and $\addw = 0$ in \Cref{corollary:finite_K}, one can obtain the result for PASGD. Comparing to previous results on non-convex objectives \cite{yu2018parallel}, we remove the uniformly bounded gradients assumption. To obtain an error bound in the form $C/\sqrt{\p K}$ for some constant $C$, our result shows $\cp$ can be large up to $\sqrt{K/\p^{3}}$ instead of $(K/\p^{3})^{1/4}$ \cite{yu2018parallel}. 


\section{Novel Analyses of Existing Algorithms}\label{sec:analyses}
Using the unified analysis of cooperative SGD presented in \Cref{thm:gen}, one can directly derive novel analyses of EASGD, PASGD and D-PSGD. The general framework also provides new insights such as the best choice of parameter $\alpha$ in EASGD (see \Cref{lem:opt_a}). 

\subsection{EASGD $\alg(1,\W_\alpha,1)$}
\label{sec:easgd}
Recall that EASGD uses hyper-parameter $\alpha$ to control the eigenvalues of mixing matrix. For $\W_\alpha$ defined in \eqref{eqn:easgd_w}, the second largest eigenvalue magnitude is 
\begin{align}
    \spgap = \max\{|1-\alpha|,|1-(\p+1)\alpha|\}. \label{eqn:easgd_zeta}
\end{align}
In order to let $\W_\alpha$ satisfy the conditions in Assumption 5, it is required that $\spgap < 1$, namely $0 \leq \alpha < 2/(\p+1)$. This condition suggests that $\alpha$ can be selected in a broader range than the original paper \citep{zhang2015deep} suggested ($0 \leq \alpha < 1/\p$).

Intuitively, a larger $\alpha$ forces more consensus between the locally trained models and improves stability. However, from equation \cref{eqn:easgd_zeta}, we observe that there exists an optimal $\alpha$ that minimizes the value of $\spgap$. 
\begin{lem}[\textbf{Best Choice of $\alpha$}]\label{lem:opt_a}
    If $\alpha = 2/(\p+2)$, then the second largest absolute eigenvalue of $\W_\alpha$, given in \Cref{eqn:easgd_zeta}, achieves the minimal value $\p/(\p+2)$.
\end{lem}
Accordingly, by choosing the best $\alpha$, the error floor at convergence can also be minimized. To be specific, we have the following theorem.
\begin{thm}[\textbf{Convergence of EASGD with the best $\alpha$}]
\label{thm:easgd}
When $\alpha$ is set to $2/(\p+2)$ as suggested by \Cref{lem:opt_a}, the error of EASGD can be bounded as follows:
\begin{align}
    \Exs\brackets{\frac{1}{K}\sum_{k=1}^K\vecnorm{\tg(\genx_k)}^2} 
        \leq& \frac{2\brackets{\F(\genx_1) - \F_\text{inf}}}{\genlr K} + \frac{\genlr\lip\V}{\p} +\nonumber \\
            &\frac{1}{2}\genlr^2\lip^2\V(\p+1)
\end{align}
where $\genx_k$ and $\genlr$ are defined in \eqref{eqn:def_gen}.
\end{thm}


To the best of our knowledge, this theorem is the first convergence result for EASGD with general objectives and also the first theoretical justification for the best choice of $\alpha$. By setting $\genlr = \frac{1}{\lip}\sqrt{\frac{\p}{K}}$, one can also obtain a finite horizon result as \Cref{corollary:finite_K}.
\begin{figure*}[t!]
    \centering
    \begin{subfigure}[t]{0.32\textwidth} 
        \centering
        \includegraphics[scale = 0.33]{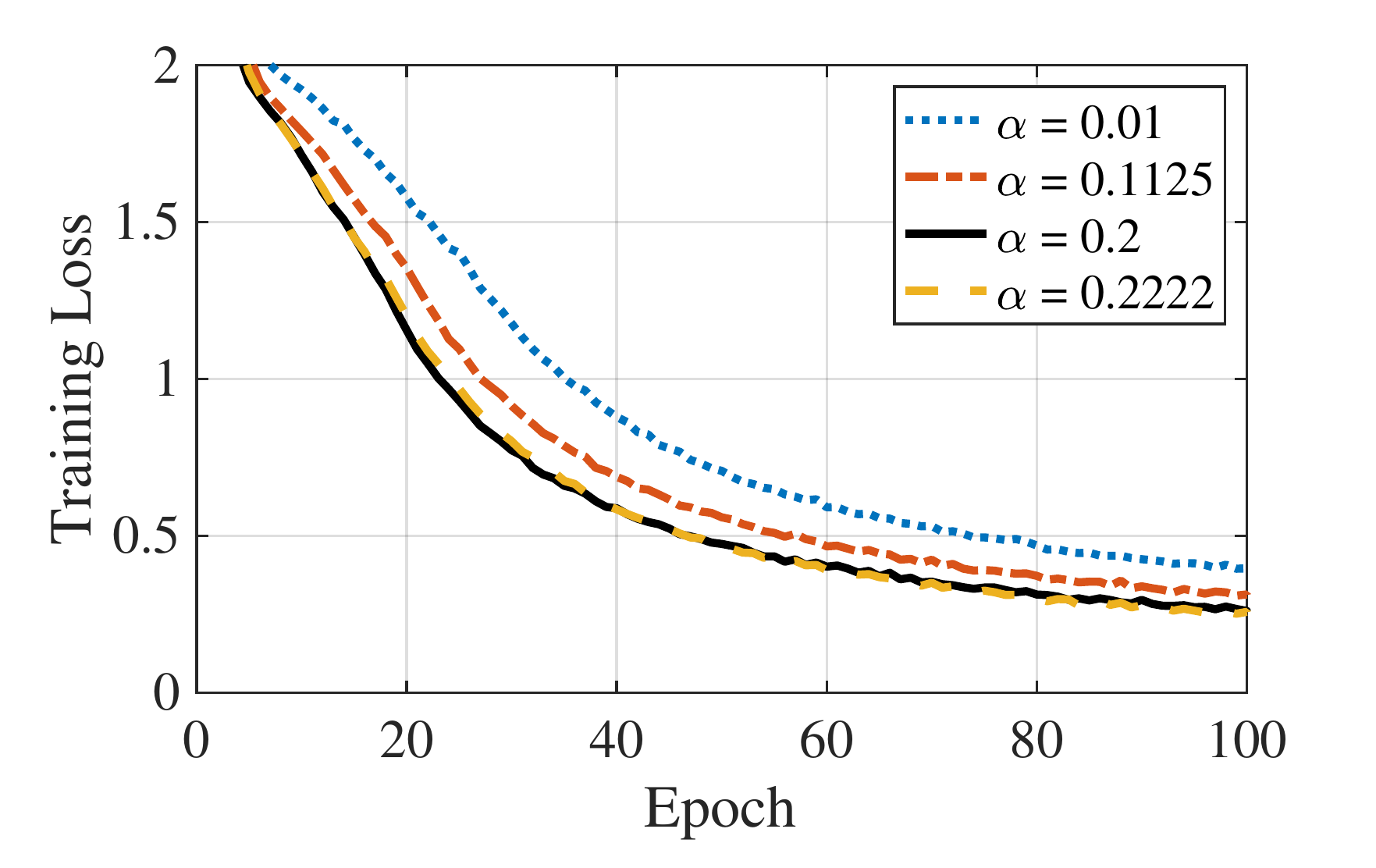}
        \caption{Average training loss of workers.}
    \end{subfigure}%
    ~
    \begin{subfigure}[t]{0.32\textwidth}
        \centering
        \includegraphics[scale = 0.33]{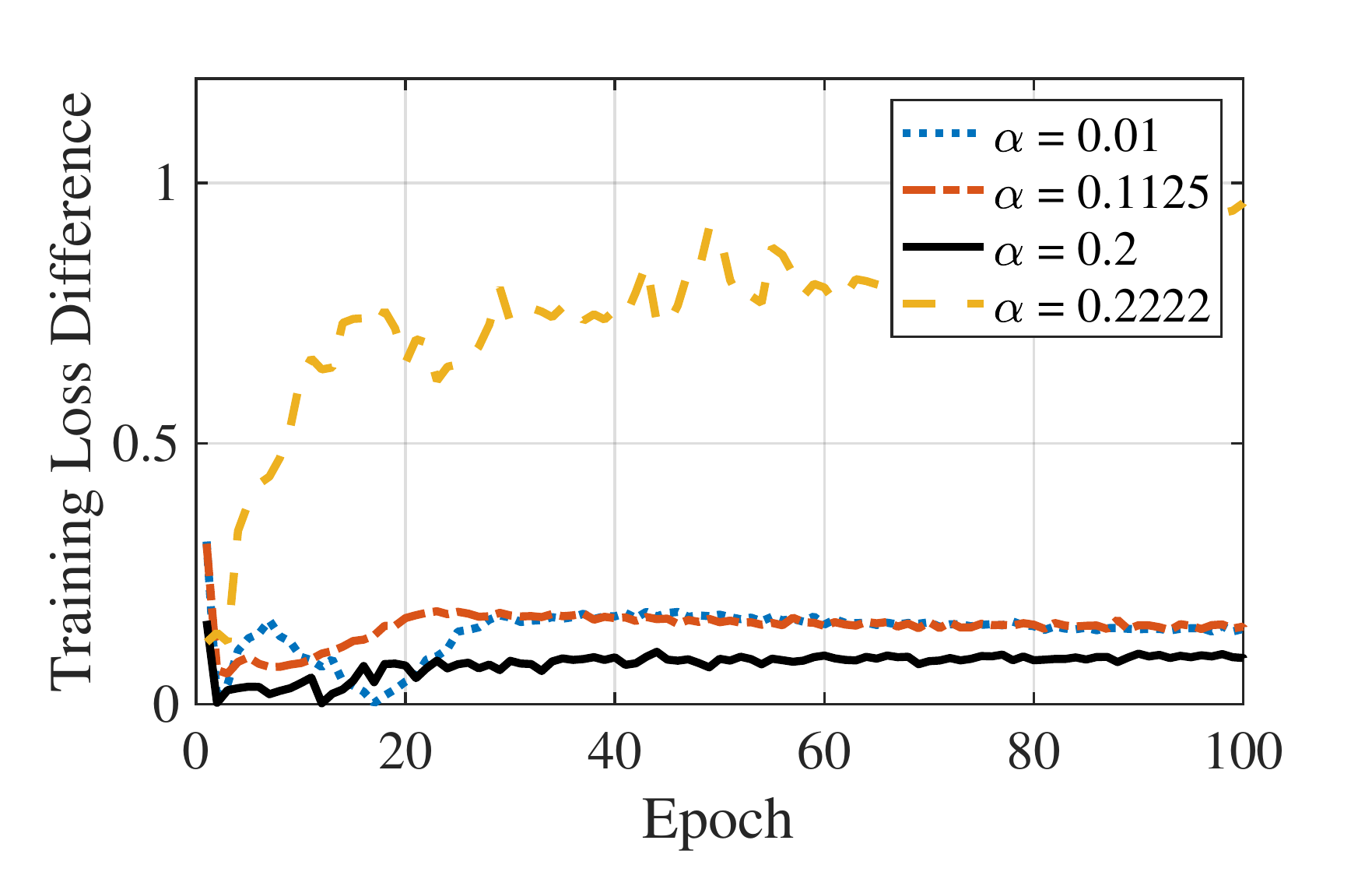}
        \caption{The difference of training loss between workers and the auxiliary variable.}
    \end{subfigure}
    ~
    \begin{subfigure}[t]{0.32\textwidth}
        \centering
        \includegraphics[scale = 0.33]{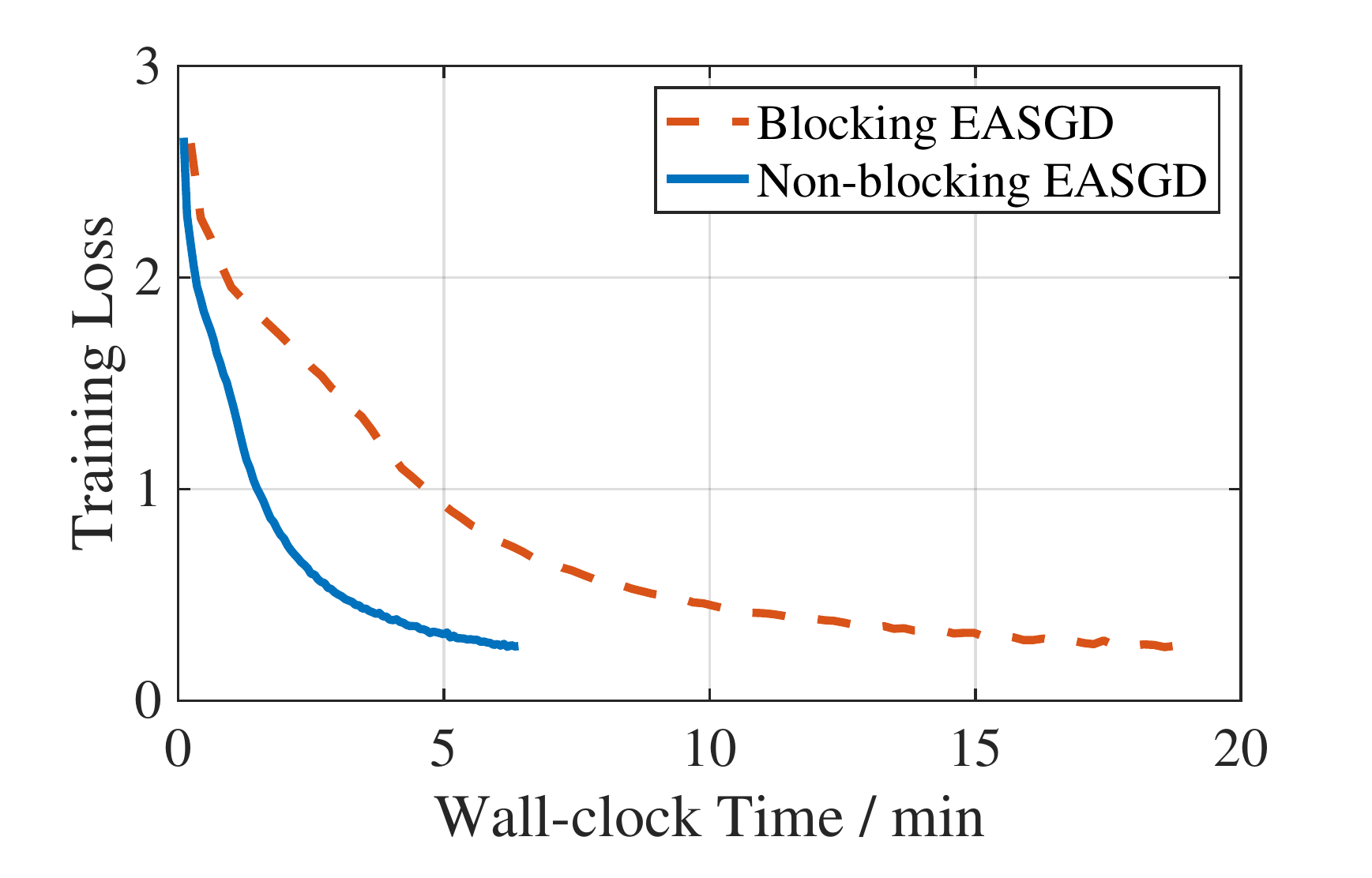}
        \caption{Benefit of non-blocking execution.}
        \label{fig:easgd_nb}
    \end{subfigure}
    \caption{EASGD training on CIFAR-10 with VGG-16. Since there are 8 worker nodes and 1 auxiliary variable, the best value of $\alpha$ given by \Cref{lem:opt_a} is $2/(\p+2) = 0.2$, which performs better than the empirical choice $\alpha = 0.9/\p = 0.1125$ suggested in \cite{zhang2015deep}. The best choice of $\alpha$ yields the lowest training loss and the least discrepancies between workers and auxiliary variable.}
    \label{fig:easgd}
\end{figure*}

\textbf{Empirical validation.} As shown in \Cref{fig:easgd}, the best choice $\alpha=2/(\p+2)=0.2$ yields fastest convergence and least discrepancies between workers and the auxiliary variable. When $\alpha$ is greater than $2/(\p+1) \approx 0.2222$, we observe the algorithm cannot converge. Furthermore, in \Cref{fig:easgd_nb}, we show the benefit of non-blocking execution. By overlapping the broadcast of auxiliary variable and workers computation, it directly reduces about $67\%$ training time. 


\subsection{PASGD $\alg(\cp,\J,0)$ Vs. D-PSGD $\alg(1,\W,0)$}\label{sec:pasgd}

The general framework enables easy comparisons between different communication reduction strategies. Here, we compare periodic communication and group synchronization strategies. Note that when PASGD $\alg(\cp,\J,0)$ and D-PSGD $\alg(1,\W,0)$ have the same error floor at convergence, we have
\begin{align}
    \frac{2\spgap_{\cp}^2}{1-\spgap_{\cp}^2} = \cp-1 \Rightarrow \spgap_{\cp} = \sqrt{1-\frac{2}{\cp+1}}. \label{eqn:zeta_th}
\end{align}
Equation \eqref{eqn:zeta_th} provides a threshold for $\spgap$. As long as $\spgap \leq \spgap_{\cp}$, D-PSGD $\alg(1,\W,0)$ would perform better than PASGD $\alg(\cp, \J, 0)$ in terms of the worst-case final error at convergence. Along with the increase of $\cp$, the value of threshold $\spgap_\cp$ rapidly converges to 1. Therefore, when $\cp$ becomes large, D-PSGD has a lower error floor in a very broad range of $\spgap$.

As for communication efficiency, the benefit of group synchronization relies on the number of workers. It at most reduces the communication overhead by $\p$ times, since at least one connection should be preserved for each worker. As the mixing matrix affects the communication delay implicitly, it is not trivial to design a good mixing matrix that not only has small eigenvalues but also enables efficient implementation. On the contrary, periodic averaging has higher flexibility without such limitations. If we set $\cp \geq \p$, then PASGD always has shorter training time than D-PSGD.

\section{Designing  New Communication-Efficient SGD Algorithms}
\label{sec:new}
As shown in \Cref{sec:analyses}, the Cooperative SGD framework enables us to analyze and compare existing communication-efficient SGD algorithms such as PASGD, EASGD and D-PSGD. The Cooperative SGD framework can also be used to design new algorithms that combine the communication-efficiency strategies adopted by these algorithms.


\subsection{Decentralized Periodic Averaging}
From \Cref{sec:pasgd} we see that D-PSGD has superior convergence performance, while PASGD can easily control the communication delay and provide higher throughput. We propose using a combination of these called decentralized periodic averaging SGD $\alg(\tau,\W,0)$ with carefully chosen $\tau$ and $\W$. For a small number of well-connected workers, larger $\tau$ is more preferable. For a large number of workers, using a sparse mixing matrix $\W$ and small $\tau$ gives better convergence. For a fixed topology worker network where $\W$ is prescribed, increasing the communication period can be an effective way to speedup the decentralized training. In \Cref{fig:decentralized_fedavg}, we implemented the algorithm with $7$ worker nodes and evaluated it on CIFAR10 dataset. The observation is decentralized periodic averaging with $\cp = 15,\spgap = 0.75$ achieves significant speedup over the pure D-PSGD algorithm as well as similar throughput as pure PASGD with a larger communication period $\cp = 50$.
\begin{figure}
    \centering
    \includegraphics[width = .43\textwidth]{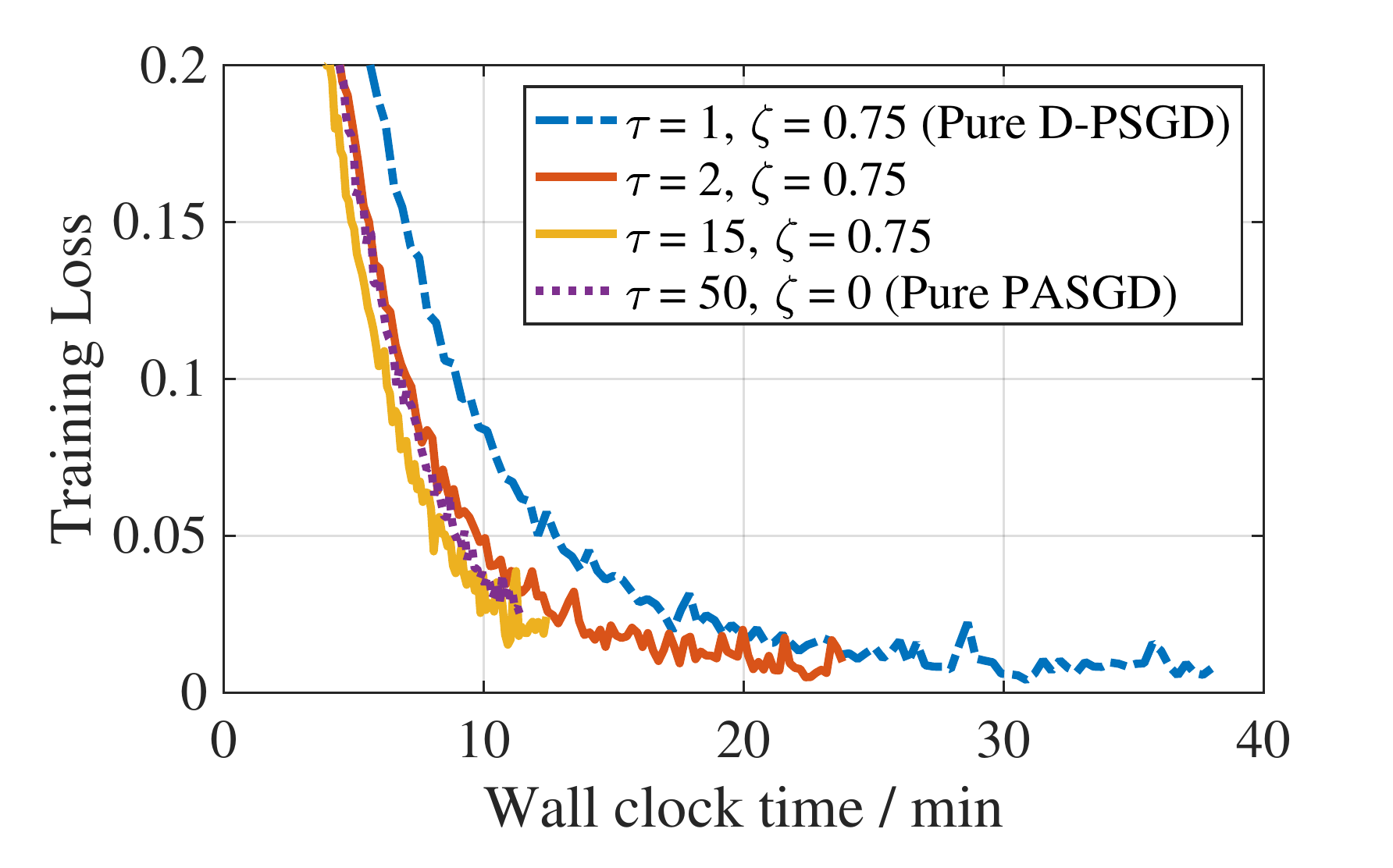}
    \caption{Decentralized periodic averaging on CIFAR-10 with VGG-16. It achieves significant speedup over pure D-PSGD and has lower training loss than pure PASGD with a large communication period.}
    \label{fig:decentralized_fedavg}
\end{figure}

\subsection{Generalized Elastic Averaging}
In generalized elastic averaging  $\alg(1,\W',1)$, we modify decentralized SGD with mixing matrix $\W$ by adding an auxiliary variable (with elasticity parameter $\alpha$) stored at a new node that is connected to all $m$ worker nodes. Recall that a sparse mixing matrix $\W$ can reduce communication delay, but it may have large $\zeta$ that leads to inferior convergence. Introducing the auxiliary variable results in the mixing matrix $\W'$ shown in \eqref{eqn:new_W} below. The second largest eigenvalue of this matrix is $(1-\alpha)$ lower than $\zeta$ as shown by \Cref{lem:reduce_zeta}.

\begin{lem}\label{lem:reduce_zeta}
    Suppose there is a $\p$-dimension symmetric matrix $\W$ such that  $\W\one = \one$, and its eigen-values satisfy $-1\leq \lambda_\p(\W)\leq \cdots \leq \lambda_1(\W)\leq 1$. Let $\zeta = \max\{|\lambda_2(\W)|,|\lambda_\p(\W)|\}$. Then, for matrix $\W'$ which is defined as:
    \begin{align}
        \W' = 
        \begin{bmatrix}
        (1-\alpha)\W & \alpha \one \\
        \alpha\one\tp & 1-\p\alpha
        \end{bmatrix}, \label{eqn:new_W}
    \end{align}
    we have 
    \begin{align}
     \spgap' &= \max\{|\lambda_2(\W')|,|\lambda_{\p+1}(\W')|\} \\ &=\max\{(1-\alpha)\spgap, |1-(\p+1)\alpha|\}.
    \end{align}
     Setting $\alpha = \frac{1+\spgap}{\p+1+\spgap}$ yields the minimum $\spgap'=\frac{\p\spgap}{\p+1+\spgap}$.
\end{lem}
The proof is given in the Appendix. 
\Cref{lem:reduce_zeta} implies that by setting $\alpha=\frac{1+\spgap}{\p+1+\spgap}$, the new algorithm $\alg(1,\W',1)$ gives a lower error bound at convergence as compared to D-PSGD $\alg(1,\W,0)$ as $\spgap' < \spgap$. 
Furthermore, since the updates and broadcast of the auxiliary variable can overlap with the local computation at workers (as explained in Section 3.3), we do not expect an increase in the training time. Thus, adding an auxiliary variable is a highly effective method to increase the consensus between loosely connected workers. 

\subsection{Hierarchical Averaging}
Based on the analysis of Cooperative SGD, we believe that a hierarchical averaging framework will aptly capture the benefits of all the communication-efficiency strategies discussed in this paper. In particular, consider that workers are divided into groups that cannot directly communicate with each other, as shown in \Cref{fig:hierarchical}(a). Local models in each group will be averaged via an auxiliary node. Inter-auxiliary node communication can occur concurrently with local updates at workers, as illustrated in \Cref{fig:hierarchical}(b). Our unified convergence analysis can be applied to this hierarchical averaging model and ongoing research includes finding the node structure that gives the best convergence.
\begin{figure}[!t]
    \centering
    \begin{subfigure}[t]{0.48\textwidth} 
        \centering
        \includegraphics[scale = 0.42]{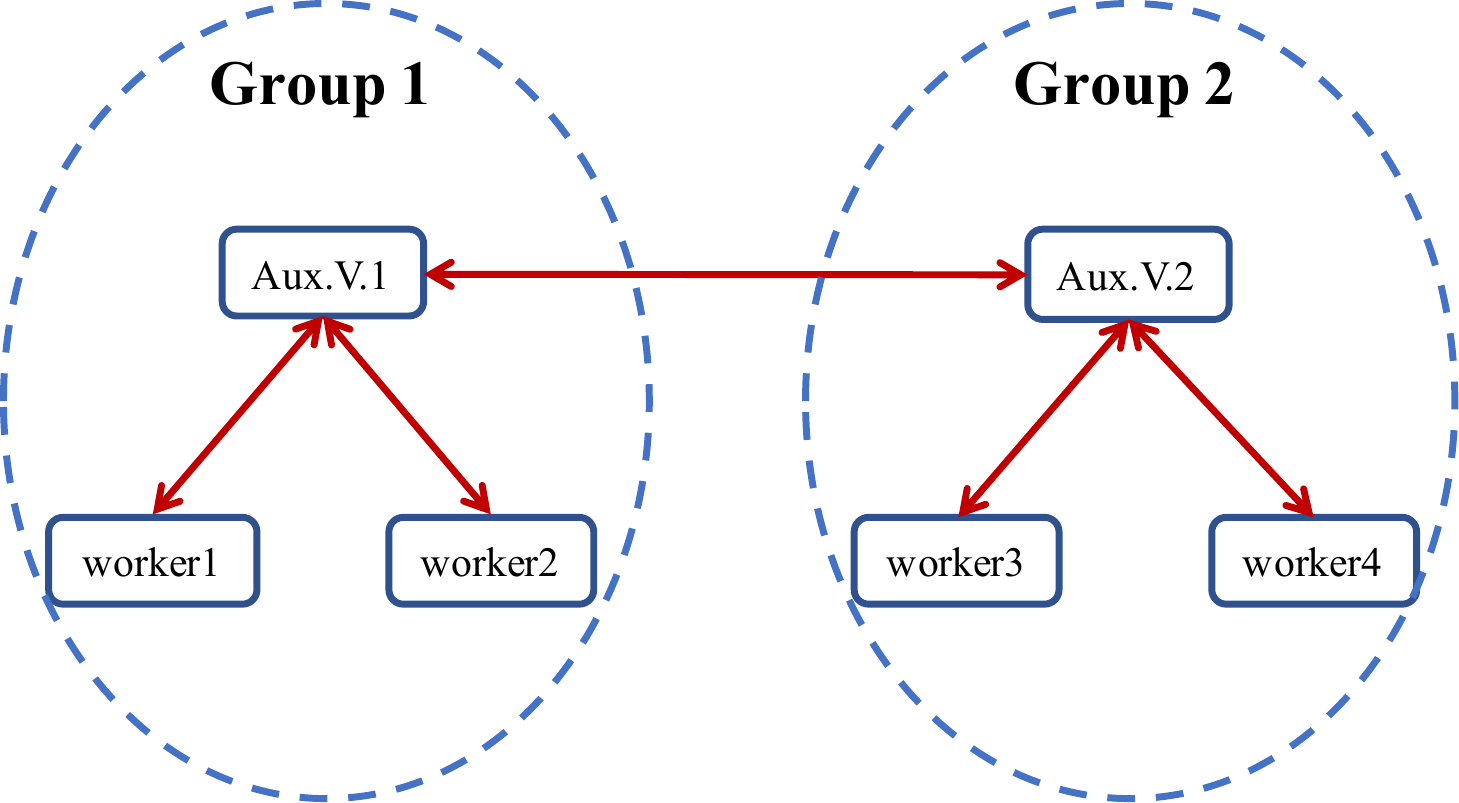}
        \caption{Hierarchical topology with two worker groups.}
    \end{subfigure}
    
    \begin{subfigure}[t]{0.48\textwidth}
        \centering
        \includegraphics[width=\columnwidth]{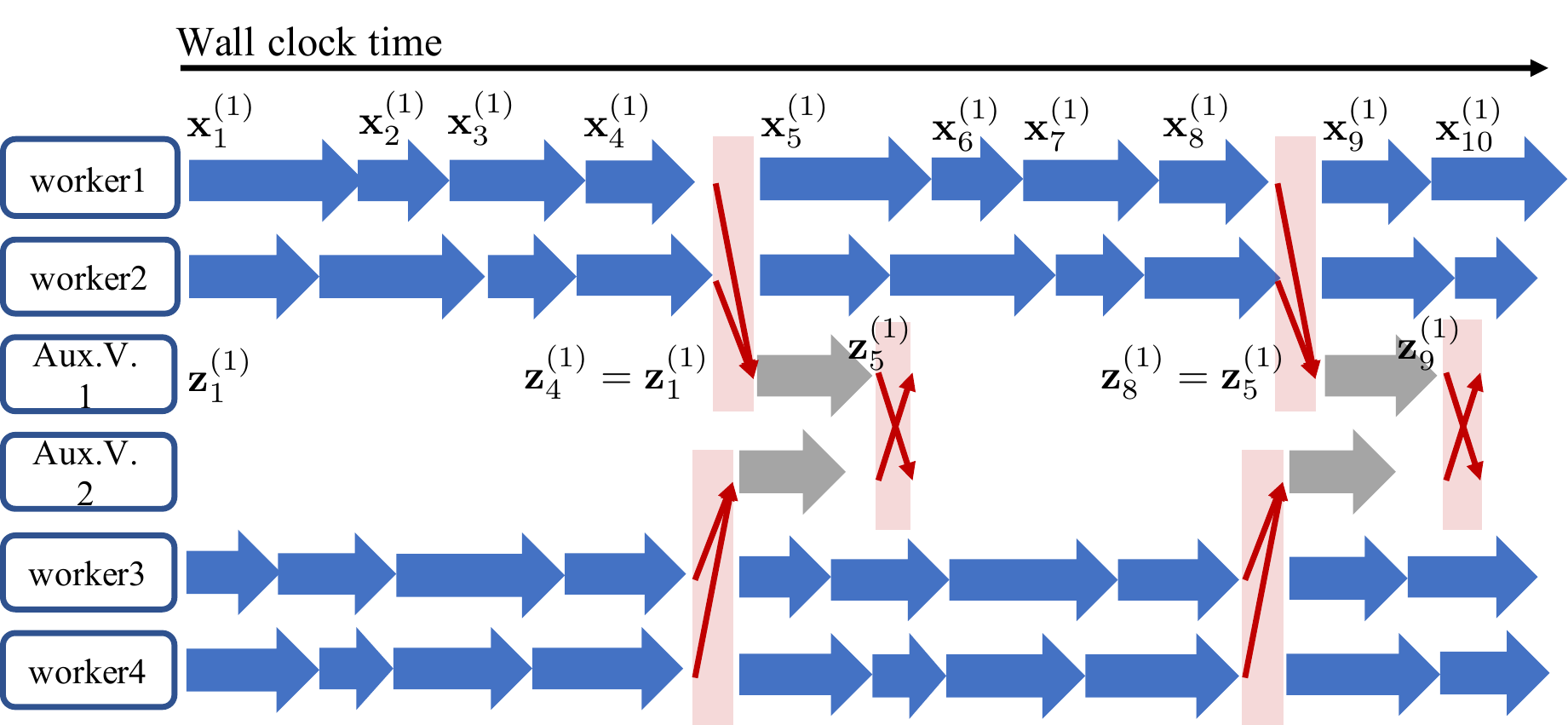}
        \caption{Potential execution timeline of a hierarchical model.}
    \end{subfigure}
    \caption{Illustration of a variant of cooperative SGD: hierarchical averaging. Blue, red, grey arrows represent gradient computation, communication among workers, and update of auxiliary variables respectively.}
    \label{fig:hierarchical}
\end{figure}

\section{Concluding Remarks}
We propose a communication-efficient SGD framework called \emph{Cooperative SGD} that combines the periodic, decentralized, and elastic model-averaging strategies to reduce inter-node communication via local updates at worker nodes. By analyzing cooperative SGD for general non-convex objectives, we provide strong convergence guarantees for existing communication-efficient SGD variants, and to the best of our knowledge, the first general analysis of elastic-averaging SGD. Furthermore, the cooperative SGD framework greatly enlarges the design space of communication-efficient SGD algorithms. We present some promising new ideas such as decentralized periodic averaging, generalized elastic-averaging and hierarchical averaging that can strike a good trade-off between convergence speed and communication efficiency. However, further exploration of the communication-efficient SGD design space and analyses of new variants is ripe for future investigation.

\section*{Acknowledgments}
The authors thank Anit Kumar Sahu for his suggestions and feedback. This work was partially supported by the CMU Dean's fellowship and an IBM Faculty Award. The experiments were conducted on the ORCA cluster provided by the Parallel Data Lab at CMU, and on Amazon AWS (supported by an AWS credit grant).

\bibliographystyle{icml2019} 
\bibliography{ref.bib}

\clearpage

\onecolumn
\appendix
\begin{center}
    \Large{\textbf{Supplemental Material}}
\end{center}


\section{Convergence of PASGD and D-PSGD}
By directly setting $\W = \J$ (i.e., $\spgap = 0$) and $\addw = 0$ in Theorem 1, one can obtain the convergence guarantee for PASGD. Comparing to previous results on non-convex objectives \citep{zhou2017convergence,yu2018parallel}, our result removes the uniformly bounded gradients assumption and provides a tighter upper bound. 
\begin{corollary}[\textbf{Convergence of PASGD}]
    For $\alg(\cp, \J, 0)$, under the same assumptions as Theorem 1, if the learning rate satisfies $\lr \lip + \lr^2 \lip^2 \cp(\cp-1) \leq 1$, then we have
    \begin{align}
        \Exs\brackets{\frac{1}{K}\sum_{k=1}^K\vecnorm{\tg(\avgx_k)}^2} 
        \leq& \frac{2\brackets{\F(\x_1) - \F_\text{inf}}}{\lr K} + \frac{\lr\lip\V}{\p}+ \lr^2\lip^2\V(\cp-1). \label{eqn:fedavg}
    \end{align}
    \label{corollary:fedavg}
\end{corollary}
The notable insight provided by \Cref{corollary:fedavg} is there exists a trade-off between the error-convergence and communication-efficiency. While a larger communication period leads to higher error at convergence, it directly reduces the communication delay by $\tau$ times and enables higher throughput. The primary advantage of PASGD is that one can easily change the communication period and find the best one that has the fastest convergence rate with respect to wall-clock time. The best value of $\cp$ should depend on the network bandwidth/latency and vary in different environments. 

\textbf{Empirical validation.} In \Cref{fig:tradeoff}, we show this trade-off in PASGD with different learning rate choices. One can see that even though PASGD with $\cp = 100$ finishes the training first, it has the highest loss after the same number of iterations. Comparing \Cref{fig:tradeoff} (a) and (b), observe that the small learning rate reduces the gap between different communication periods. This phenomenon has already been discussed in Theorem 1: small learning rate can alleviate the relative effect of the network error term. Besides, for completeness, we present the test accuracy of PASGD in \Cref{fig:tradeoff} (c). The interesting observation is that PASGD with large communication period has better generalization performance than fully synchronous SGD.

\begin{figure*}[!ht]
    \centering
    \begin{subfigure}[t]{0.32\textwidth} 
        \centering
        \includegraphics[scale = 0.33]{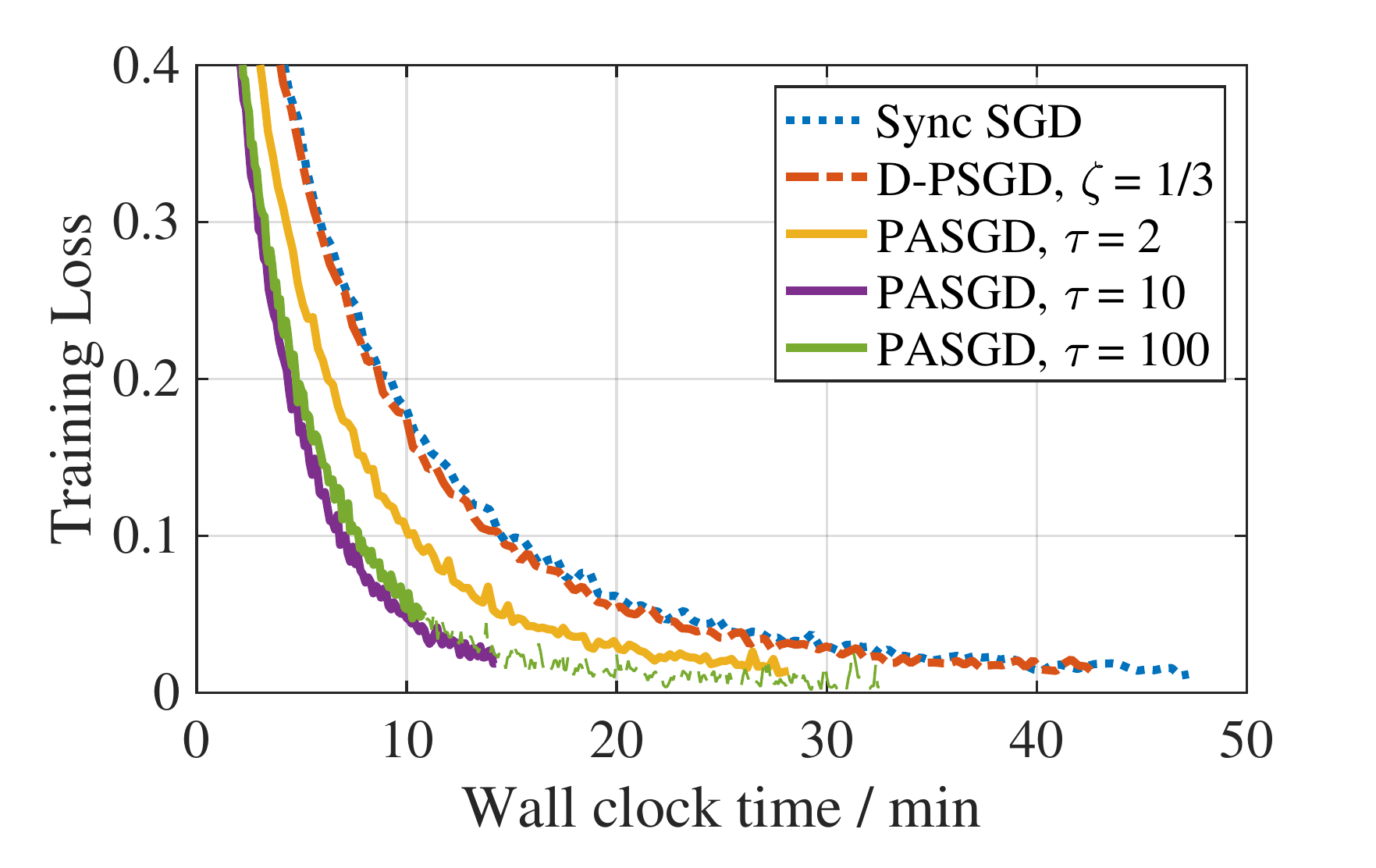}
        \caption{Learning rate equals to $0.02$.}
    \end{subfigure}%
    ~
    \begin{subfigure}[t]{0.32\textwidth}
        \centering
        \includegraphics[scale = 0.33]{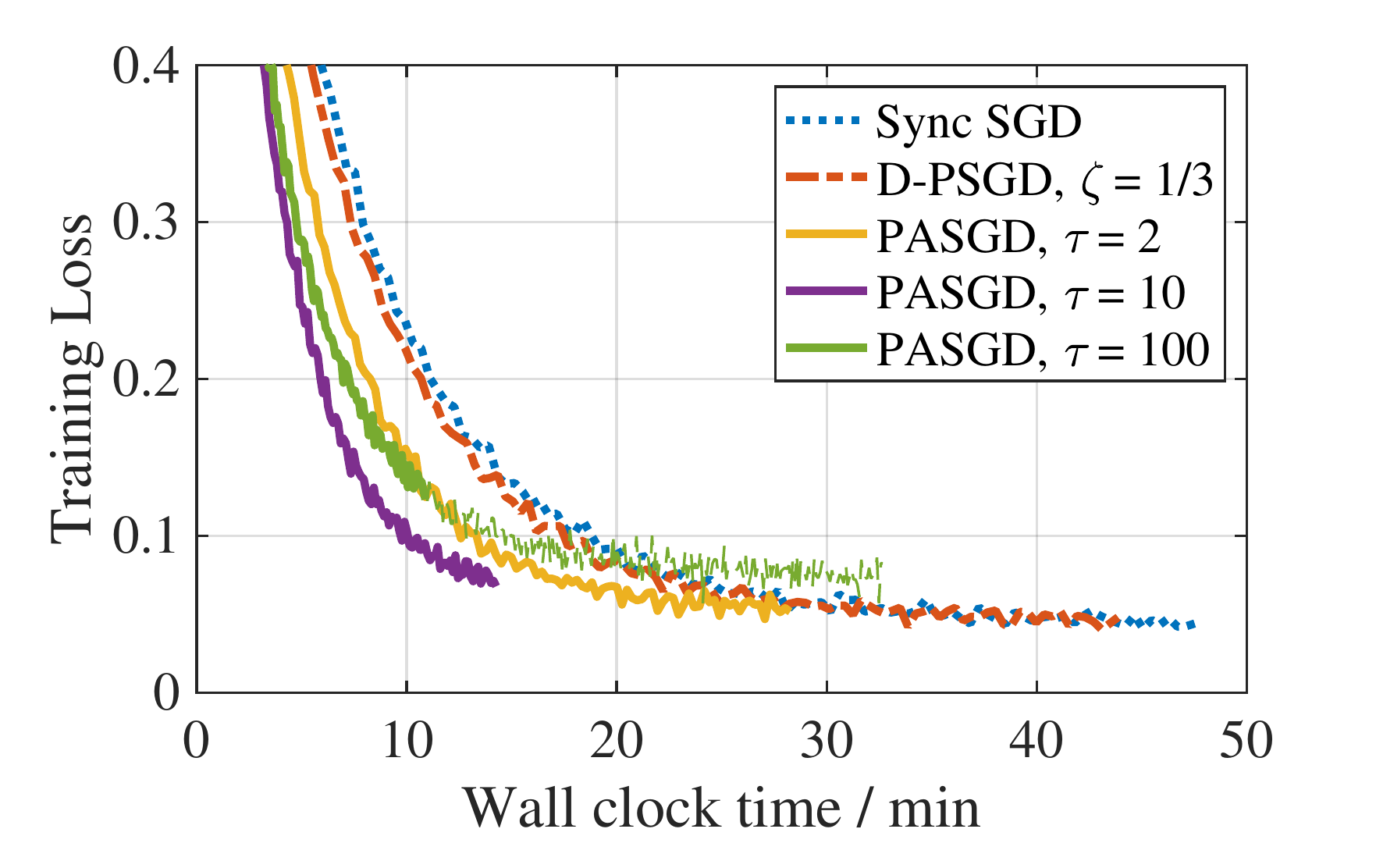}
        \caption{Learning rate equals to $0.2$.}
    \end{subfigure}
    ~
    \begin{subfigure}[t]{0.32\textwidth}
        \centering
        \includegraphics[scale = 0.33]{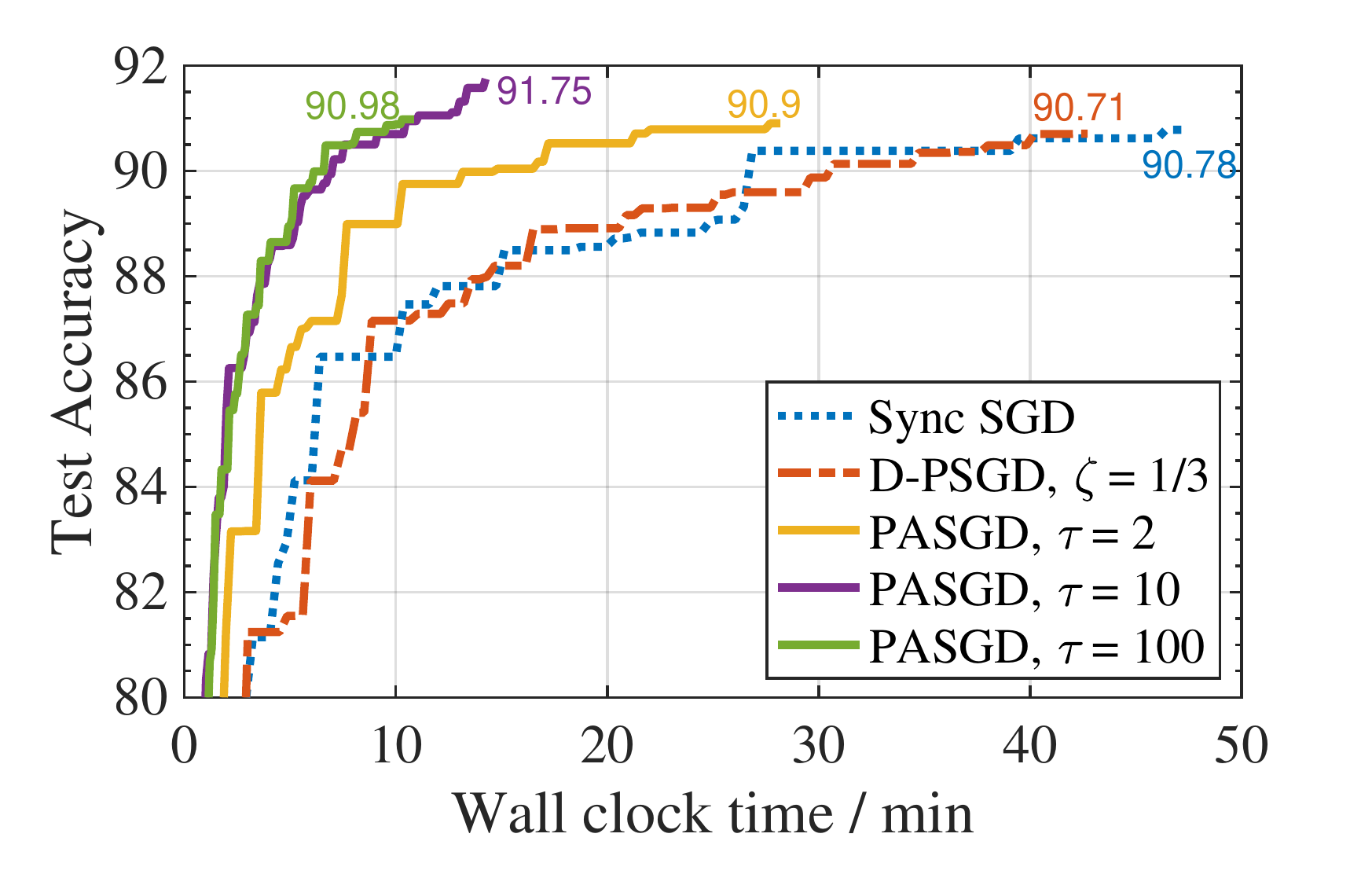}
        \caption{Test accuracy versus wall clock time. Learning rate equals to $0.02$.}
    \end{subfigure}
    \caption{Illustration of error-convergence and communication-efficiency trade-off in PASGD. Each line was trained for $120$ epochs. As for PASGD with $\cp = 100$, we trained it for $360$ epochs to show the error floor at convergence. The extra training curve is shown as the green dash line. After the same number of epochs, a larger communication period leads to higher training loss but costs much less wall clock time.} 
    \label{fig:tradeoff}
\end{figure*}

Similarly, setting $\cp = 1$ and $\addw = 0$ in Theorem 1, we get the convergence guarantee for D-PSGD, which is consistent to \cite{lian2017can}.
\begin{corollary}[\textbf{Convergence of D-PSGD}]
    For $\alg(1, \W, 0)$, under the same assumptions as Theorem 1, if the learning rate satisfies
    \begin{align}
        \lr\lip + \lr^2\lip^2\frac{2\spgap}{1-\spgap} \parenth{\frac{\spgap}{1+\spgap}+\frac{1}{1-\spgap}} \leq 1, 
    \end{align}
    where $\spgap = \max\{|\lambda_2(\W)|,|\lambda_\p(\W)|\}$, then we have
    \begin{align}
        \Exs\brackets{\frac{1}{K}\sum_{k=1}^K\vecnorm{\tg(\avgx_k)}^2} 
        \leq& \frac{2\brackets{\F(\x_1) - \F_\text{inf}}}{\lr K} + \frac{\lr\lip\V}{\p}+ \lr^2\lip^2\V\frac{2\spgap^2}{1-\spgap^2}.  \label{eqn:dsgd}
    \end{align}
    \label{corollary:dsgd}
\end{corollary}
We implemented a ring-connected D-PSGD with $4$ workers. As shown in \Cref{fig:tradeoff}, the number of workers is so few that the communication reduction effect is quite limited (saving about $10\%$ training time while PASGD ($\cp=10$) reduces about $67\%$ time).

\section{Proof Preliminaries}
For the ease of writing, we first define some notations. Let $\Xi_k$ denote the set $\{\minib_k^{(1)},\dots,\minib_k^{(\p)}\}$ of mini-batches at $\p$ workers in iteration $k$. We use notation $\CExs$ to denote the conditional expectation $\Exs_{\Xi_K | \X_k}$. Besides, define averaged stochastic gradient and averaged full batch gradient as follows:
\begin{align}
    \gsg_k = \frac{1}{\p}\summp \sg(\x_k^{(i)}), \ \gtg_k = \frac{1}{\p}\summp \tg(\x_k^{(i)}). \label{eqn:def_g_h}
\end{align}
Similar to $\X_k$ and $\G_k$, we stack all full batch gradients in a $d \times (\p+\addw)$ dimension matrix:
\begin{align}
    \tg(\X_k) = [\tg(\x_k^{(1)}), \dots, \tg(\x_k^{(\p)}), \mathbf{0}, \dots, \mathbf{0}].
\end{align}
Accordingly, the Frobenius norm of full batch gradients is $\fronorm{\tg(\X_k)}^2 = \summp \vecnorm{\tg(\x_k^{(i)})}^2$. In order to facilitate reading, the definitions of matrix Frobenius norm and operator norm are also provided here.
\begin{defn}[\cite{horn1990matrix}]
The Frobenius norm defined for $\mathbf{A} \in M_n$ by
\begin{align}
    \fronorm{\mathbf{A}}^2 = |\trace(\mathbf{A}\mathbf{A}\tp)| = \sum_{i,j=1}^n |a_{ij}|^2.
\end{align}
\end{defn}
\begin{defn}[\cite{horn1990matrix}]
The operator norm defined for $\mathbf{A} \in M_n$ by
\begin{align}
    \opnorm{\mathbf{A}} = \max_{\vecnorm{\mathbf{x}}=1}\vecnorm{\mathbf{A}\mathbf{x}}=\sqrt{\lambda_{\text{max}}(\matA\tp\matA)}.
\end{align}
\end{defn}
All notations used in the proof are listed below.
\begin{table}[!h]
    \begin{center}
    \begin{tabular}{l|c}
    \toprule
        Number of workers & $\p$ \\
        Number of auxiliary variables & $\addw$ \\
        Total iterations & $K$ \\
        Communication period & $\cp$ \\
        Mixing matrix & $\W$ \\
        Learning rate & $\lr$ \\
        Lipschitz constant & $\lip$ \\
        Variance bounds for stochastic gradients & $\M,\V$ \\
    \bottomrule
    \end{tabular}
    \end{center}
    \caption{List of notations.}
    \label{tab:nots}
\end{table}

\section{A Supporting Lemma for Theorem 1}
Before providing the proof of Theorem 1, we prefer to first present an important lemma that describes the basic intuition for the convergence of cooperative SGD: the discrepancies of local models have a negative impact on the convergence. The proof of Theorem 1 will be built upon this lemma.
\setcounter{lem}{2}
\begin{lem}[\textbf{Error decomposition}]
    For algorithm $\alg(\cp,\W,\addw)$, under Assumption 1--5, if the learning rate satisfies $\genlr \lip(1+\M/\p) \leq 1$ and all local model parameters are initialized at the same point $\x_1$, then the average-squared gradient after $K$ iterations is bounded as follows
    \begin{align}
        \Exs\brackets{\frac{1}{K}\sum_{k=1}^K\vecnorm{\tg(\genx_k)}^2} 
        &\leq \underbrace{\frac{2\brackets{\F(\x_1) - \F_\text{inf}}}{\genlr K} + \frac{\genlr\lip\V}{\p}}_{\text{fully sync SGD}} + \underbrace{\frac{\lip^2}{K} \sum_{k=1}^K\frac{\Exs\fronorm{\X_k(\I-\J)}^2}{\p}}_{\text{network error}} \label{eqn:basic}
    \end{align}
    where $\genx_k,\genlr$ are defined in (13) and both $\I$ and $\J$ are $(\p+\addw)\times(\p+\addw)$ matrices.
    \label{lem:basic}
\end{lem}

\subsection{Proof of \Cref{lem:basic}}
\subsubsection{Lemmas}
\begin{lem}\label{lem:avg_var}
    Under Assumption 3 and 4, we have the following variance bound for the averaged stochastic gradient:
    \begin{align}
        \Exs_{\Xi_K | \X_k}\brackets{\vecnorm{\gsg_k - \gtg_k}^2} &\leq \frac{\M}{\p^2}\fronorm{\tg(\X_k)}^2 + \frac{\V}{\p}.
    \end{align}
\end{lem}
\begin{proof}
According to the definition of $\gsg_k,\gtg_k$ \Cref{eqn:def_g_h}, we have
\begin{align}
    &\Exs_{\Xi_K | \X_k}\brackets{\vecnorm{\gsg_k - \gtg_k}^2} \\
    =& \Exs_{\Xi_K | \X_k}\vecnorm{\frac{1}{\p}\summp\brackets{\sg(\x_k^{(i)}) - \tg(\x_k^{(i)})}}^2 \\
    =& \frac{1}{\p^2}\Exs_{\Xi_K | \X_k}\brackets{\summp\vecnorm{\sg(\x_k^{(i)}) - \tg(\x_k^{(i)})}^2+\sum_{j \neq l}^{\p}\inprod{\sg(\x_k^{(j)}) - \tg(\x_k^{(j)})}{\sg(\x_k^{(l)}) - \tg(\x_k^{(l)})}}  \label{eqn:pre_1} \\
    =& \frac{1}{\p^2}\summp\Exs_{\minib_k^{(i)}|\X_k}\vecnorm{\sg(\x_k^{(i)}) - \tg(\x_k^{(i)})}^2 + \frac{1}{\p^2}\sum_{j \neq l}^{\p}\inprod{\Exs_{\minib_k^{(j)}|\X_k}\brackets{\sg(\x_k^{(j)}) - \tg(\x_k^{(j)})}}{\Exs_{\minib_k^{(l)}|\X_k}\brackets{\sg(\x_k^{(l)}) - \tg(\x_k^{(l)})}} \label{eqn:pre_2}
\end{align}
where equation \eqref{eqn:pre_2} is due to $\{\minib_k^{(i)}\}$ are independent random variables. Now, directly applying Assumption 3 and 4 to \eqref{eqn:pre_2}, one can observe that all cross terms are zero. Then, we have
\begin{align}
    \Exs_{\Xi_K | \X_k}\vecnorm{\gsg_k - \gtg_k}^2
    \leq& \frac{1}{\p^2}\summp\brackets{\M\vecnorm{\tg(\x_k^{(i)})}^2+\V} \\
    =& \frac{\M}{\p}\frac{\fronorm{\tg(\X_k)}^2}{\p} + \frac{\V}{\p}.
\end{align}
\end{proof}

\begin{lem}\label{lem:p1}
    Under Assumption 3, the expected inner product between stochastic gradient and full batch gradient can be expanded as
    \begin{align}
        \CExs\brackets{\inprod{\tg(\genx_k)}{\gsg_k}}
        = \frac{1}{2}\vecnorm{\tg(\genx_k)}^2 + \frac{1}{2\p} \summp \vecnorm{\tg(\x_k^{(i)})}^2 - \frac{1}{2\p} \summp \vecnorm{\tg(\genx_k) - \tg(\x_k^{(i)})}^2
    \end{align}
    where $\CExs$ denotes the conditional expectation $\Exs_{\Xi_K | \X_k}$.
\end{lem}
\begin{proof}
\begin{align}
    \CExs\brackets{\inprod{\tg(\genx_k)}{\gsg_k}}
    &= \CExs\brackets{\inprod{\tg(\genx_k)}{\frac{1}{\p}\summp \sg(\x_k^{(i)})}} \\
    &= \frac{1}{\p} \summp \inprod{\tg(\genx_k)}{\tg(\x_k^{(i)})}  \\
    &= \frac{1}{2\p} \summp \brackets{\vecnorm{\tg(\genx_k)}^2 + \vecnorm{\tg(\x_k^{(i)})}^2 - \vecnorm{\tg(\genx_k) - \tg(\x_k^{(i)})}^2} \label{eqn:inprod_bnd_1} \\
    &= \frac{1}{2}\vecnorm{\tg(\genx_k)}^2 + \frac{1}{2\p} \summp \vecnorm{\tg(\x_k^{(i)})}^2 - \frac{1}{2\p} \summp \vecnorm{\tg(\genx_k) - \tg(\x_k^{(i)})}^2 
\end{align}
where equation \eqref{eqn:inprod_bnd_1} comes from $2\mathbf{a}\tp\mathbf{b} = \vecnorm{\mathbf{a}}^2 + \vecnorm{\mathbf{b}}^2 - \vecnorm{\mathbf{a}-\mathbf{b}}^2$.
\end{proof}

\begin{lem}\label{lem:p2}
    Under Assumption 3 and 4, the squared norm of stochastic gradient can be bounded as
    \begin{align}
        \CExs\brackets{\vecnorm{\gsg_k}^2}
        \leq \parenth{\frac{\M}{\p}+1}\frac{\fronorm{\tg(\X_k)}^2}{\p} + \frac{\V}{\p}. \nonumber
    \end{align}
\end{lem}
\begin{proof}
Since $\CExs [\gsg_k] = \gtg_k$, then we have
\begin{align}
    \CExs \brackets{\vecnorm{\gsg_k}^2}
    &=  \CExs \brackets{\vecnorm{\gsg_k - \CExs [\gsg_k]}^2} + \vecnorm{\CExs [\gsg_k]}^2 \\
    &= \CExs\brackets{\vecnorm{\gsg_k-\gtg_k}^2} + \vecnorm{\gtg_k}^2  \\
    &\leq \frac{\M}{\p}\frac{\fronorm{\tg(\X_k)}^2}{\p} + \frac{\V}{\p} + \vecnorm{\gtg_k}^2 \label{eqn:g_k_bnd_1}\\
    &\leq \frac{\M}{\p}\frac{\fronorm{\tg(\X_k)}^2}{\p} + \frac{\V}{\p} + \frac{1}{\p}\fronorm{\tg(\X_k)}^2 \label{eqn:g_k_bnd_2}\\
    &=  \parenth{\frac{\M}{\p}+1}\frac{\fronorm{\tg(\X_k)}^2}{\p} + \frac{\V}{\p}, 
\end{align}
where \eqref{eqn:g_k_bnd_1} follows \Cref{lem:avg_var} and \eqref{eqn:g_k_bnd_2} comes from the convexity of vector norm and Jensen's inequality:
\begin{align}
    \vecnorm{\gtg_k}^2
    = \vecnorm{\frac{1}{\p}\summp\tg(\x_k^{(i)})}^2
    \leq \frac{1}{\p}\summp\vecnorm{\tg(\x_k^{(i)})}^2
    = \frac{1}{\p}\fronorm{\tg(\X_k)}^2.
\end{align}
\end{proof}

\subsubsection{Proof of \Cref{lem:basic}}
According to Lipschitz continuous gradient assumption, we have
\begin{align}
    \CExs\brackets{\F(\genx_{k+1})} - \F(\genx_k)
    &\leq -\genlr\CExs\brackets{\inprod{\tg(\genx_k)}{\gsg_k}} + \frac{\genlr^2\lip}{2}\CExs \brackets{\vecnorm{\gsg_k}^2}.
\end{align}
Combining with \Cref{lem:p1,lem:p2}, we obtain
\begin{align}
    \CExs\brackets{\F(\genx_{k+1})} - \F(\genx_k)
    \leq& -\frac{\genlr}{2}\vecnorm{\tg(\genx_k)}^2 - \frac{\genlr}{2\p} \summp \vecnorm{\tg(\x_k^{(i)})}^2 + \frac{\genlr}{2\p} \summp \vecnorm{\tg(\genx_k) - \tg(\x_k^{(i)})}^2+ \nonumber \\
        & \frac{\genlr^2\lip}{2\p}\summp \vecnorm{\tg(\x_k^{(i)})}^2 \cdot \parenth{\frac{\M}{\p}+1} + \frac{\genlr^2\lip\V}{2\p}  \\
    \leq& -\frac{\genlr}{2}\vecnorm{\tg(\genx_k)}^2 - \frac{\genlr}{2}\brackets{1-\genlr \lip\parenth{\frac{\M}{\p}+1}}\cdot\frac{1}{\p} \summp \vecnorm{\tg(\x_k^{(i)})}^2+  \nonumber \\
        & \frac{\genlr^2\lip\V}{2\p} + \frac{\genlr\lip^2}{2\p} \summp \vecnorm{\genx_k - \x_k^{(i)}}^2.
\end{align}
After minor rearranging and according to the definition of Frobenius norm, it is easy to show
\begin{align}
    \vecnorm{\tg(\genx_k)}^2 
    \leq& \frac{2\brackets{\F(\genx_k) - \CExs\brackets{ \F(\genx_{k+1})}}}{\genlr}  + \frac{\genlr\lip\V}{\p} + \frac{\lip^2}{\p} \summp \vecnorm{\genx_k - \x_k^{(i)}}^2- \nonumber \\
        & \brackets{1-\genlr \lip\parenth{\frac{\M}{\p}+1}}\frac{1}{\p} \fronorm{\tg(\X_k)}^2.
\end{align}
Taking the total expectation and averaging over all iterates, we have
\begin{align}
    \Exs\brackets{\frac{1}{K}\sum_{k=1}^K\vecnorm{\tg(\genx_k)}^2} 
    \leq& \frac{2\brackets{\F(\genx_1) - \F_\text{inf}}}{\genlr K} + \frac{\genlr\lip\V}{\p} + \frac{\lip^2}{K\p} \sum_{k=1}^K\summp \Exs \vecnorm{\genx_k - \x_k^{(i)}}^2- \nonumber \\
        & \brackets{1-\genlr \lip\parenth{\frac{\M}{\p}+1}}\frac{1}{K}\sum_{k=1}^K\frac{\Exs\fronorm{\tg(\X_k)}^2}{\p}. \label{eqn:tight_thm1}
\end{align}
If the effective learning rate satisfies $\genlr\lip(\M/\p+1) \leq 1$, then
\begin{align}
    \Exs\brackets{\frac{1}{K}\sum_{k=1}^K\vecnorm{\tg(\genx_k)}^2} 
    \leq& \frac{2\brackets{\F(\genx_1) - \F_\text{inf}}}{\genlr K} + \frac{\genlr\lip\V}{\p} + \frac{\lip^2}{K\p} \sum_{k=1}^K\summp \Exs \vecnorm{\genx_k - \x_k^{(i)}}^2. \label{eqn:mid_1}
\end{align}
Recalling the definition $\genx_k = \X_k\one_{\p+\addw}/(\p+\addw)$ and adding a positive term to the RHS, one can get
\begin{align}
    \summp \vecnorm{\genx_k - \x_k^{(i)}}^2 
    \leq& \summp \vecnorm{\genx_k - \x_k^{(i)}}^2 + \sum_{j=1}^\addw \vecnorm{\genx_k - \z_k^{(j)}}^2 \\
    =& \fronorm{ \genx \one_{\p+\addw}\tp - \X_k}^2 \\
    =& \fronorm{ \X_k \frac{\one_{\p+\addw} \one_{\p+\addw}\tp}{\p+\addw} - \X_k}^2 = \fronorm{\X_k (\I-\J)}^2 \label{eqn:ne_bnd_1}
\end{align}
where $\I,\J$ are $(\p+\addw)\times(\p+\addw)$ matrices. Plugging the inequality \eqref{eqn:ne_bnd_1} into \eqref{eqn:mid_1}, we complete the proof.

\section{Proof of Theorem 1: Convergence of Cooperative SGD}
\subsection{Lemmas}
\begin{lem}\label{lem:op_fro_norm}
    Consider two real matrices $\matA \in \mathbb{R}^{d \times \p}$ and $\matB \in \mathbb{R}^{\p \times \p}$. If $\matB$ is symmetric, then we have
    \begin{align}
        \fronorm{\matA\matB} \leq \opnorm{\matB}\fronorm{\matA}.
    \end{align}
\end{lem}
\begin{proof}
Assume the rows of matirx $\matA$ are denoted by $\mathbf{a}_1\tp, \dots,\mathbf{a}_d\tp$ and $\mathcal{I} = \{i \in [1,d]: \vecnorm{\mathbf{a}_i}\neq 0\}$. Then, we have
\begin{align}
    \fronorm{\matA\matB}^2
    =& \sum_{i=1}^d \vecnorm{\mathbf{a}_i\tp \matB}^2 = \sum_{i \in \mathcal{I}}^d \vecnorm{\matB\mathbf{a}_i}^2 \\
    =& \sum_{i \in \mathcal{I}}^d \frac{\vecnorm{\matB\mathbf{a}_i}^2}{\vecnorm{\mathbf{a}_i}^2}\vecnorm{\mathbf{a}_i}^2 \\
    \leq& \sum_{i \in \mathcal{I}}^d \opnorm{\matB}^2\vecnorm{\mathbf{a}_i}^2 = \opnorm{\matB}^2\sum_{i \in \mathcal{I}}^d\vecnorm{\mathbf{a}_i}^2 = \opnorm{\matB}^2\fronorm{\matA}^2
\end{align}
where the last inequality follows the definition of matrix operator norm.
\end{proof}

\begin{lem} \label{lem:tr_bnd}
Suppose there are two matrices $\matA \in \mathbb{R}^{m \times n}$ and $\matB \in \mathbb{R}^{n \times m}$. Then, we have
\begin{align}
    |\trace(\matA\matB)| \leq \fronorm{\matA}\fronorm{\matB}.
\end{align}
\end{lem}
\begin{proof}
Assume $\mathbf{a}_i\tp \in \mathbb{R}^n$ is the $i$-th row of matrix $\matA$ and $\mathbf{b}_i \in \mathbb{R}^n$ is the $i$-th column of matrix $\matB$. According to the definition of matrix trace, we have
\begin{align}
    \trace(\matA\matB)
    =& \sum_{i=1}^m \sum_{j=1}^n \matA_{ij}\matB_{ji} \\
    =& \sum_{i=1}^m \mathbf{a}_i\tp \mathbf{b}_i.
\end{align}
Then, Cauchy-Schwartz inequality yields
\begin{align}
    |\sum_{i=1}^m \mathbf{a}_i\tp \mathbf{b}_i|^2
    \leq& \parenth{\sum_{i=1}^m \vecnorm{\mathbf{a}_i}^2}\parenth{\sum_{i=1}^m \vecnorm{\mathbf{b}_i}^2} \\
    =& \fronorm{\matA}^2\fronorm{\matB}^2.
\end{align}
\end{proof}

\begin{lem}\label{lem:opnorm}
    Suppose there is a $\p \times \p$ matrix $\W$ that satisfies Assumption 5. Then
    \begin{align}
        \opnorm{\W^j-\J} = \spgap^j
    \end{align}
    where $\spgap = \max\{|\lambda_2(\W)|,|\lambda_\p(\W)|\}$.
\end{lem}
\begin{proof}
Since $\W$ is a real symmetric matrix, then it can be decomposed as $\W = \mathbf{Q}\mathbf{\Lambda}\mathbf{Q}\tp$, where $\mathbf{Q}$ is an orthogonal matrix and $\mathbf{\Lambda} = \diag\{\lambda_1(\W),\lambda_2(\W),\dots,\lambda_\p(\W)\}$. In particular, since the largest eigenvalue of $\W$ is $1$ and $\W\one = \one$, the corresponding eigenvector (i.e., the first column of $\mathbf{Q}$) is $\frac{\one}{\sqrt{\p}}$. Similarly, matrix $\J$ can be decomposed as $\mathbf{Q}\mathbf{\Lambda}_0\mathbf{Q}\tp$ where $\mathbf{\Lambda}_0 = \diag\{1,0,\dots,0\}$. Then, we have
\begin{align}
    \W^j-\J = (\mathbf{Q}\mathbf{\Lambda}\mathbf{Q}\tp)^j-\J = \mathbf{Q}\parenth{\mathbf{\Lambda}^j-\mathbf{\Lambda}_0}\mathbf{Q}\tp.
\end{align}
According to the definition of matrix operator norm,
\begin{align}
    \opnorm{\W^j -\J} = \sqrt{\lambda_{\text{max}}((\W^j-\J)\tp(\W^j-\J))} = \sqrt{ \lambda_{\text{max}}(\W^{2j}-\J)}.
\end{align}
Since $\W^{2j}-\J = \mathbf{Q}\parenth{\mathbf{\Lambda}^{2j}-\mathbf{\Lambda}_0}\mathbf{Q}\tp$, the maximal eigenvalue will be $\max\{0,\lambda_2(\W)^{2j},\dots,\lambda_\p(\W)^{2j}\}=\spgap^{2j}$. As a consequence, we have $\opnorm{\W^j-\J} =\sqrt{ \lambda_{\text{max}}(\W^{2j}-\J)}=\spgap^j$.

\end{proof}

\subsection{Proof of Theorem 1}
Recall the intermediate result \eqref{eqn:tight_thm1} in the proof of \Cref{lem:basic}:
\begin{align}
    \Exs\brackets{\frac{1}{K}\sum_{k=1}^K\vecnorm{\tg(\genx_k)}^2} 
    \leq& \frac{2\brackets{\F(\genx_1) - \F_\text{inf}}}{\genlr K} + \frac{\genlr\lip\V}{\p} + \frac{\lip^2}{K\p} \sum_{k=1}^K\fronorm{\X_k(\I-\J)}^2- \nonumber \\
        & \brackets{1-\genlr \lip\parenth{\frac{\M}{\p}+1}}\frac{1}{K}\sum_{k=1}^K\frac{\Exs\fronorm{\tg(\X_k)}^2}{\p}.
\end{align}
Our goal is to provide an upper bound for the network error term $\frac{\lip^2}{K\p} \sum_{k=1}^K\fronorm{\X_k(\I-\J)}^2$. First of all, let us derive a specific expression for $\X_k(\I-\J)$.
\subsubsection{Decomposition.}
According to the update rule (10) in Section 3, one can observe that
\begin{align}
    \X_k(\I-\J)
    =& \parenth{\X_{k-1} - \lr \G_{k-1}}\W_{k-1}(\I-\J) \\
    =& \X_{k-1}(\I-\J)\W_{k-1} - \lr\G_{k-1}(\W_{k-1}-\J) \label{eqn:x_k_p_bnd_1}
\end{align}
where \eqref{eqn:x_k_p_bnd_1} follows the special property of doubly stochastic matrix: $\W_{k-1}\J = \J\W_{k-1} = \J$ and hence $(\I-\J)\W_{k-1} = \W_{k-1}(\I-\J)$. Then, expanding the expression of $\X_{k-1}$, we have
\begin{align}
    \X_{k}(\I-\J)
    =& \brackets{\X_{k-2}(\I-\J)\W_{k-2} - \lr\G_{k-2}(\W_{k-2}-\J)}\W_{k-1} - \lr\G_{k-1}(\W_{k-1}-\J) \\
    =& \X_{k-2}(\I-\J)\W_{k-2}\W_{k-1} - \lr\G_{k-2}(\W_{k-2}\W_{k-1}-\J) - \lr\G_{k-1}(\W_{k-1}-\J)
\end{align}
Repeating the same procedure for $\X_{k-2},\X_{k-3},\dots,\X_2$, finally we get
\begin{align}
    \X_{k}(\I-\J)
    =& \X_1(\I-\J)\matPhi_{1,k-1} -\lr \sum_{s=1}^{k-1}\G_s(\matPhi_{s,k-1}-\J)
\end{align}
where $\matPhi_{s,k-1} = \prod_{l=s}^{k-1} \W_l$. Since all optimization variables are initialized at the same point $\X_1(\I-\J) = 0$, the squared norm of the network error term can be directly written as
\begin{align}
    \Exs\fronorm{\X_{k}(\I-\J)}^2
    =& \lr^2 \Exs\fronorm{\sum_{s=1}^{k-1}\G_s(\matPhi_{s,k-1}-\J)}^2.
\end{align}

Then, let us take a closer look at the expression of $\matPhi_{s,k-1}$. Without loss of generality, assume $k = j\cp+i$, where $j$ denotes the index of communication rounds and $i$ denotes the index of local updates. As a result, matrix $\matPhi_{s,k-1}$ can be expressed as follows:
\begin{align}
    \matPhi_{s,k-1}
    =
    \begin{cases}
    \I, & j\cp < s < j\cp + i \\
    \W, & (j-1)\cp < s \leq j\cp \\
    \W^2, & (j-2)\cp < s \leq (j-1)\cp \\
    \cdots \\
    \W^j, & 0<s \leq \cp
    \end{cases}.
\end{align}

For the ease of writing, define accumulated stochastic gradient within one local update period as $\msg_r = \sum_{s=r\cp+1}^{(r+1)\cp}\G_s$ for $0\leq r<j$ and $\msg_j = \sum_{s=j\cp+1}^{j\cp+i-1}\G_s$. Similarly, define accumulated full batch gradient $\mtg_r = \sum_{s=r\cp+1}^{(r+1)\cp}\tg(\X_s)$ for $0\leq r<j$ and $\mtg_j = \sum_{s=j\cp+1}^{j\cp+i-1}\tg(\X_s)$. Accordingly, we have
\begin{align}
    &\sum_{s=1}^{\cp}\G_s(\matPhi_{s,k-1} -\J) = \msg_0 (\W^j - \J), \\
    &\sum_{s=\cp+1}^{2\cp}\G_s(\matPhi_{s,k-1} -\J) = \msg_1 (\W^{j-1} - \J), \\
    &\dots \\
    &\sum_{s=j\cp+1}^{j\cp+i-1}\G_s(\matPhi_{s,k-1} -\J) = \msg_j(\I-\J).
\end{align}
Thus, summing all these terms we get
\begin{align}
    \sum_{s=1}^{k-1}\G_s(\matPhi_{s,k-1} -\J) = \sum_{r=0}^j \msg_r (\W^{j-r}-\J).
\end{align}
Note that the network error term can be decomposed into two parts:
\begin{align}
    \Exs\fronorm{\X_{k}(\I-\J)}^2
    =& \lr^2 \Exs\fronorm{\sum_{r=0}^j \msg_r (\W^{j-r}-\J)}^2 \\
    =& \lr^2 \Exs\fronorm{\sum_{r=0}^j (\msg_r - \mtg_r) (\W^{j-r}-\J) + \sum_{r=0}^j \mtg_r (\W^{j-r}-\J)}^2 \\
    \leq& \underbrace{2\lr^2 \Exs\fronorm{\sum_{r=0}^j (\msg_r-\mtg_r) (\W^{j-r}-\J)}^2}_{T_1} +\underbrace{2\lr^2\Exs\fronorm{\sum_{r=0}^j \mtg_r (\W^{j-r}-\J)}^2}_{T_2} \label{eqn:x_k_p_decomp}
\end{align}
where \Cref{eqn:x_k_p_decomp} follows $\vecnorm{a+b}^2 \leq 2\vecnorm{a}^2 + 2\vecnorm{b}^2$. Next, we are going to separately provide bounds for $T_1$ and $T_2$. Recall that we are interested in the average of all iterates $\frac{\lip^2}{K\p} \sum_{k=1}^K\fronorm{\X_k(\I-\J)}^2$. Accordingly, we will also derive the bounds for $\frac{\lip^2}{K\p}\sum_{k=1}^K T_1$ and $\frac{\lip^2}{K\p}\sum_{k=1}^K T_2$.


\subsubsection{Bounding $T_1$.}
For the first term $T_1$, we have
\begin{align}
    T_1 
    =& 2\lr^2 \sum_{r=0}^j\Exs\fronorm{(\msg_r-\mtg_r) (\W^{j-r}-\J)}^2 \label{eqn:t_1_bnd_1} \\
    \leq& 2\lr^2 \sum_{r=0}^j\Exs\fronorm{\msg_r-\mtg_r}^2\opnorm{\W^{j-r}-\J}^2 \label{eqn:t_1_bnd_2}\\
    =& 2\lr^2 \sum_{r=0}^j\Exs\fronorm{\msg_r-\mtg_r}^2\spgap^{2(j-r)} \label{eqn:t_1_bnd_op}\\
    =& 2\lr^2 \sum_{r=0}^{j-1}\Exs\fronorm{\msg_r-\mtg_r}^2\spgap^{2(j-r)} + 2\lr^2 \Exs\fronorm{\msg_j-\mtg_j}^2 \label{eqn:t_1_decomp}
\end{align}
where \eqref{eqn:t_1_bnd_2} follows \Cref{lem:op_fro_norm}, \eqref{eqn:t_1_bnd_op} comes from \Cref{lem:opnorm}. Recall that $\spgap = \max\{|\lambda_2(\W)|, |\lambda_{\p+\addw}(\W)|\}$. Then for any $0 \leq r < j$,
\begin{align}
    \Exs\brackets{\fronorm{\msg_r -\mtg_r}^2}
    =& \Exs\brackets{\fronorm{\sum_{s=r\cp+1}^{(r+1)\cp}\brackets{\G_s - \tg(\X_s)}}^2} \\
    =& \summp \Exs \brackets{\vecnorm{\sum_{s=r\cp+1}^{(r+1)\cp}\brackets{\sg(\x_s^{(i)}) - \tg(\x_s^{(i)})}}^2} \\
    =& \summp \Exs\brackets{\sum_{s=r\cp+1}^{(r+1)\cp}\vecnorm{\sg(\x_s^{(i)}) - \tg(\x_s^{(i)})}^2 + \sum_{s \neq l}\inprod{\sg(\x_s^{(i)}) - \tg(\x_s^{(i)})}{\sg(\x_l^{(i)}) - \tg(\x_l^{(i)})}}.
\end{align}
Now we show that the cross terms are zero. For any $s < l$, according to Assumption 4, one can obtain
\begin{align}
    &\Exs \brackets{\inprod{\sg(\x_s^{(i)}) - \tg(\x_s^{(i)})}{\sg(\x_l^{(i)}) - \tg(\x_l^{(i)})}} \nonumber \\
    =& \Exs_{\x_{s}^{(i)},\minib_{s}^{(i)},\x_{l}^{(i)}}\Exs_{\minib_l^{(i)}|\x_{s}^{(i)},\minib_{s}^{(i)},\x_{l}^{(i)}}\brackets{\inprod{\sg(\x_s^{(i)}) - \tg(\x_s^{(i)})}{\sg(\x_l^{(i)}) - \tg(\x_l^{(i)})}} \\
    =& \Exs\brackets{\inprod{\sg(\x_s^{(i)}) - \tg(\x_s^{(i)})}{\Exs_{\minib_l^{(i)}|\x_{s}^{(i)},\minib_{s}^{(i)},\x_{l}^{(i)}}\brackets{\sg(\x_l^{(i)}) - \tg(\x_l^{(i)})}}} \\
    =& \Exs\brackets{\inprod{\sg(\x_s^{(i)}) - \tg(\x_s^{(i)})}{0}} = 0.
\end{align}
As a result, we have
\begin{align}
    \Exs\brackets{\fronorm{\msg_r -\mtg_r}^2}
    =& \Exs \brackets{\sum_{s=r\cp+1}^{(r+1)\cp} \summp\vecnorm{\sg(\x_s^{(i)}) - \tg(\x_s^{(i)})}^2} \\
    \leq& \M \sum_{s=r\cp+1}^{(r+1)\cp}\summp\Exs\brackets{\vecnorm{\tg(\x_s^{(i)})}^2} + \cp\p \V \label{eqn:y_q_bnd_1}\\
    =& \M \sum_{s=r\cp+1}^{(r+1)\cp}\fronorm{\tg(\X_s)}^2 + \cp\p \V \label{eqn:t_1_part1}
\end{align}
where \eqref{eqn:y_q_bnd_1} is according to Assumption 4. Using the same technique, one can obtain that 
\begin{align}
    \Exs\fronorm{\msg_j - \mtg_j}^2 
    \leq& \M \sum_{s=j\cp+1}^{j\cp+i-1}\fronorm{\tg(\X_s)}^2 + (i-1)\p\V. \label{eqn:t_1_part2}
\end{align}
Substituting \Cref{eqn:t_1_part1,eqn:t_1_part2} back into \eqref{eqn:t_1_decomp}, we have
\begin{align}
    T_1
    \leq& 2\lr^2 \sum_{r=0}^{j-1}\brackets{\spgap^{2(j-r)}\parenth{\M \sum_{s=r\cp+1}^{(r+1)\cp}\fronorm{\tg(\X_s)}^2 +\cp\p\V}} + 2\lr^2\M \sum_{s=j\cp+1}^{j\cp+i-1}\fronorm{\tg(\X_s)}^2 + 2\lr^2(i-1)\p\V \\
    \leq& 2\lr^2\p\V\brackets{\frac{\spgap^2}{1-\spgap^2}\cp + i-1} + 2\lr^2 \M \sum_{r=0}^{j-1}\brackets{\spgap^{2(j-r)}\parenth{\sum_{s=r\cp+1}^{(r+1)\cp}\fronorm{\tg(\X_s)}^2}} + 2\lr^2\M \sum_{s=j\cp+1}^{j\cp+i-1}\fronorm{\tg(\X_s)}^2 \label{eqn:t_1_bnd_3}
\end{align}
where \eqref{eqn:t_1_bnd_3} follows the summation formula of power series:
\begin{align}
    \sum_{r=0}^{j-1}\spgap^{2(j-r)} 
    \leq \sum_{r = -\infty}^{j-1}\spgap^{2(j-r)} \leq \frac{\spgap^2}{1-\spgap^2}.
\end{align}
Next, summing over all iterates in the $j$-th local update period (from $i = 1$ to $i = \cp$):
\begin{align}
    \sum_{i=1}^\cp T_1
    \leq& \lr^2\p\V\brackets{\frac{2\spgap^2}{1-\spgap^2}\cp^2 + \cp(\cp-1)}+ 2\lr^2 \M\cp \sum_{r=0}^{j-1}\brackets{\spgap^{2(j-r)}\parenth{\sum_{s=r\cp+1}^{(r+1)\cp}\fronorm{\tg(\X_s)}^2}} + 2\lr^2\M\cp \sum_{s=j\cp+1}^{(j+1)\cp-1}\fronorm{\tg(\X_s)}^2 \\
    \leq& \lr^2\p\V\brackets{\frac{2\spgap^2}{1-\spgap^2}\cp^2 + \cp(\cp-1)}+ 2\lr^2 \M\cp \sum_{r=0}^{j}\brackets{\spgap^{2(j-r)}\parenth{\sum_{s=r\cp+1}^{(r+1)\cp}\fronorm{\tg(\X_s)}^2}}.
\end{align}
Then, summing over all periods from $j = 0$ to $j = K/\cp - 1$, where $K$ is the total iterations:
\begin{align}
    \sum_{j=0}^{K/\cp-1}\sum_{i=1}^\cp T_1
    \leq& \frac{K}{\cp}\lr^2\p\V\brackets{\frac{2\spgap^2}{1-\spgap^2}\cp^2 + \cp(\cp-1)} + 2\lr^2 \M\cp\sum_{j=0}^{K/\cp-1} \sum_{r=0}^{j}\brackets{\spgap^{2(j-r)}\parenth{\sum_{s=r\cp+1}^{(r+1)\cp}\fronorm{\tg(\X_s)}^2}} \\
    =& K\lr^2\p\V\brackets{\frac{1+\spgap^2}{1-\spgap^2}\cp -1} + 2\lr^2 \M\cp\sum_{j=0}^{K/\cp-1} \sum_{r=0}^{j}\brackets{\spgap^{2(j-r)}\parenth{\sum_{s=r\cp+1}^{(r+1)\cp}\fronorm{\tg(\X_s)}^2}}. \label{eqn:t_1_final2}
\end{align}
Expanding the summation in \eqref{eqn:t_1_final2}, we have
\begin{align}
    \sum_{j=0}^{K/\cp-1}\sum_{i=1}^\cp T_1
    \leq& K\lr^2\p\V\brackets{\frac{1+\spgap^2}{1-\spgap^2}\cp -1} + 2\lr^2 \M\cp\sum_{r=0}^{K/\cp-1}\brackets{ \parenth{\sum_{s=r\cp+1}^{(r+1)\cp}\fronorm{\tg(\X_s)}^2}\parenth{\sum_{j=r}^{K/\cp-1}\spgap^{2(j-r)}}} \\
    \leq& K\lr^2\p\V\brackets{\frac{1+\spgap^2}{1-\spgap^2}\cp -1} + 2\lr^2 \M\cp\sum_{r=0}^{K/\cp-1}\brackets{ \parenth{\sum_{s=r\cp+1}^{(r+1)\cp}\fronorm{\tg(\X_s)}^2}\parenth{\sum_{j=r}^{+\infty}\spgap^{2(j-r)}}} \\
    \leq& K\lr^2\p\V\brackets{\frac{1+\spgap^2}{1-\spgap^2}\cp -1} + \frac{2\lr^2 \M\cp}{1-\spgap^2}\sum_{r=0}^{K/\cp-1} \parenth{\sum_{s=r\cp+1}^{(r+1)\cp}\fronorm{\tg(\X_s)}^2} \\
    =& K\lr^2\p\V\brackets{\frac{1+\spgap^2}{1-\spgap^2}\cp -1} + \frac{2\lr^2 \M\cp}{1-\spgap^2}\sum_{k=1}^K \fronorm{\tg(\X_k)}^2. \label{eqn:t_1_final}
\end{align}
Here, we complete the first part.


\subsubsection{Bounding $T_2$.}
For the second term in \eqref{eqn:x_k_p_decomp}, since $\fronorm{\matA}^2 = \trace(\matA\tp \matA)$, we have
\begin{align}
    T_2
    =& 2\lr^2\sum_{r=0}^j \Exs\fronorm{\mtg_r (\W^{j-r}-\J)}^2 +2\lr^2\sum_{n=0}^j\sum_{l = 0, l\neq n}^j\Exs \brackets{\trace\parenth{(\W^{j-n}-\J)\mtg_n\tp\mtg_l(\W^{j-l}-\J)}}.
\end{align}
According to \Cref{lem:tr_bnd}, the trace can be bounded as:
\begin{align}
    |\trace\parenth{(\W^{j-n}-\J)\mtg_n\tp\mtg_l(\W^{j-l}-\J)}|
    \leq&  \fronorm{(\W^{j-n}-\J)\mtg_n\tp} \fronorm{\mtg_l(\W^{j-l}-\J)} \\
    \leq& \opnorm{\W^{j-n}-\J}\fronorm{\mtg_n} \fronorm{\mtg_l}\opnorm{\W^{j-l}-\J} \label{eqn:tr_bnd_1} \\
    \leq& \frac{1}{2}\spgap^{2j-n-l}\brackets{\fronorm{\mtg_n}^2+\fronorm{\mtg_l}^2} \label{eqn:tr_bnd_2}
\end{align}
where \eqref{eqn:tr_bnd_1} follows \Cref{lem:op_fro_norm} and \eqref{eqn:tr_bnd_2} is because of $2ab \leq a^2 + b^2$. Then, it follows that
\begin{align}
    T_2
    \leq& 2\lr^2\sum_{r=0}^j \Exs\fronorm{\mtg_r}^2\opnorm{(\W^{j-r}-\J)}^2 +\lr^2\sum_{n=0}^j\sum_{l=0,l\neq n}^j\spgap^{2j-n-l}\Exs\brackets{\fronorm{\mtg_n}^2+\fronorm{\mtg_l}^2} \\
    =& 2\lr^2\sum_{r=0}^j \spgap^{2(j-r)}\Exs\fronorm{\mtg_r}^2 +2\lr^2\sum_{n=0}^j\sum_{l=0,l\neq n}^j\spgap^{2j-n-l}\Exs\fronorm{\mtg_n}^2 \label{eqn:t2_bnd_1}\\
    =& 2\lr^2\sum_{r=0}^j \spgap^{2(j-r)}\Exs\fronorm{\mtg_r}^2 +2\lr^2\sum_{n=0}^j\spgap^{j-n}\Exs\fronorm{\mtg_n}^2\sum_{l=0,l\neq n}^j\spgap^{j-l}\\
    =& 2\lr^2\brackets{\sum_{r=0}^{j-1} \spgap^{2(j-r)}\Exs\fronorm{\mtg_r}^2 +\sum_{n=0}^{j-1}\spgap^{j-n}\Exs\fronorm{\mtg_n}^2\sum_{l=0,l\neq n}^j\spgap^{j-l} + \Exs\fronorm{\mtg_j}^2 + \Exs\fronorm{\mtg_j}^2\sum_{l=0}^{j-1}\spgap^{j-l}} \\
    \leq& 2\lr^2\brackets{\sum_{r=0}^{j-1} \spgap^{2(j-r)}\Exs\fronorm{\mtg_r}^2 +\sum_{n=0}^{j-1}\frac{\spgap^{j-n}}{1-\spgap}\Exs\fronorm{\mtg_n}^2+\Exs\fronorm{\mtg_j}^2+ \Exs\fronorm{\mtg_j}^2\frac{\spgap}{1-\spgap}} \label{eqn:t2_bnd_power}
\end{align}
where \eqref{eqn:t2_bnd_1} uses the fact that indices $n$ and $l$ are symmetric and \eqref{eqn:t2_bnd_power} is according to the summation formula of power series: 
\begin{align}
    \sum_{l=0,l\neq n}^j\spgap^{j-l} \leq \sum_{l=-\infty}^j\spgap^{j-l} \leq \frac{1}{1-\spgap}, \\ \sum_{l=0}^{j-1}\spgap^{j-l} \leq \sum_{l=-\infty}^{j-1}\spgap^{j-l} \leq \frac{\spgap}{1-\spgap}.
\end{align}
After minor rearranging, we have
\begin{align}
    T_2
    \leq& 2\lr^2\sum_{r=0}^{j-1}\brackets{ \parenth{\spgap^{2(j-r)}+\frac{\spgap^{j-r}}{1-\spgap}}\Exs\fronorm{\mtg_r}^2} +\frac{2\lr^2}{1-\spgap}\Exs\fronorm{\mtg_j}^2\\
    =& 2\lr^2\sum_{r=0}^{j-1} \brackets{\parenth{\spgap^{2(j-r)} + \frac{\spgap^{j-r}}{1-\spgap}}\Exs\fronorm{\sum_{s=1}^\cp \tg(\X_{r\cp+s})}^2} + \frac{2\lr^2}{1-\spgap}\Exs\fronorm{\sum_{s=1}^{i-1}\tg(\X_{j\cp+s})}^2 \\
    \leq& 2\lr^2\cp\sum_{r=0}^{j-1} \brackets{\parenth{\spgap^{2(j-r)} + \frac{\spgap^{j-r}}{1-\spgap}}\sum_{s=1}^\cp\Exs\fronorm{\tg(\X_{r\cp+s})}^2} + \frac{2\lr^2(i-1)}{1-\spgap}\sum_{s=1}^{i-1}\Exs\fronorm{\tg(\X_{j\cp+s})}^2. \label{t2_bnd_2}
\end{align}
where \eqref{t2_bnd_2} follows the convexity of Frobenius norm and Jensen's inequality. Next, summing over all iterates in the $j$-th period, we can get 
\begin{align}
    \sum_{i=1}^\cp T_2
    \leq& 2\lr^2\cp^2\sum_{r=0}^{j-1} \brackets{\parenth{\spgap^{2(j-r)} + \frac{\spgap^{j-r}}{1-\spgap}}\sum_{s=1}^\cp\Exs\fronorm{\tg(\X_{r\cp+s})}^2} + \lr^2\cp(\cp-1)\frac{1}{1-\spgap}\sum_{s=1}^{\cp-1}\Exs\fronorm{\tg(\X_{j\cp+s})}^2. \label{eqn:t2_bnd_3}
\end{align}
Now, we are going to provide a bound for the summation over all periods (from $j = 0$ to $j = \crounds -1$). For clarity, let us first focus on the $r$-th local update period ($r<j$). The coefficient of $\sum_{s=1}^\cp \Exs\vecnorm{\tg(\X_{r\cp+s})}^2$ in \eqref{eqn:t2_bnd_3} is
\begin{align}
    \spgap^{2(j-r)} + \frac{\spgap^{j-r}}{1-\spgap}.
\end{align}
Accordingly, the coefficient of $\sum_{s=1}^\cp \Exs\vecnorm{\tg(\X_{r\cp+s})}^2$ in $\sum_{j=0}^{\crounds-1}\sum_{i=1}^\cp T_2$ can be written as:
\begin{align}
    \sum_{j=r+1}^{\crounds-1}\parenth{\spgap^{2(j-r)} + \frac{\spgap^{j-r}}{1-\spgap}}
    \leq& \sum_{j=r+1}^{\infty}\parenth{\spgap^{2(j-r)} + \frac{\spgap^{j-r}}{1-\spgap}} \\
    \leq& \frac{\spgap^2}{1-\spgap^2} + \frac{\spgap}{(1-\spgap)^2}.
\end{align}
As a result, we have
\begin{align}
    \sum_{j=0}^{\crounds-1}\sum_{i=1}^\cp T_2
    \leq& 2\lr^2\cp^2 \parenth{\frac{\spgap^2}{1-\spgap^2}+\frac{\spgap}{(1-\spgap)^2}}\sum_{j=1}^{\crounds-1}\sum_{s=1}^\cp\Exs\fronorm{\tg(\X_{j\cp+s})}^2 + \frac{\lr^2\cp(\cp-1)}{1-\spgap}\sum_{j=0}^{\crounds-1}\sum_{s=1}^{\cp-1}\Exs\fronorm{\tg(\X_{j\cp+s})}^2
\end{align}
Replacing all indices by $k$,
\begin{align}
    \sum_{j=0}^{\crounds-1}\sum_{i=1}^\cp T_2
    \leq& 2\lr^2\cp^2 \parenth{\frac{\spgap^2}{1-\spgap^2}+\frac{\spgap}{(1-\spgap)^2}}\sum_{k=1}^K\Exs\fronorm{\tg(\X_k)}^2 + \frac{\lr^2\cp(\cp-1)}{1-\spgap}\sum_{k=1}^{K}\Exs\fronorm{\tg(\X_k)}^2 \\
    =& \frac{\lr^2\cp^2}{1-\spgap} \parenth{\frac{2\spgap^2}{1+\spgap}+\frac{2\spgap}{1-\spgap} + \frac{\cp-1}{\cp}}\sum_{k=1}^K\Exs\fronorm{\tg(\X_k)}^2. \label{eqn:t_2_final}
\end{align}
We complete the second part.

\subsubsection{Final result.}
According to \Cref{eqn:x_k_p_decomp,eqn:t_1_final,eqn:t_2_final}, the network error can be bounded as
\begin{align}
    \frac{1}{K\p}\sum_{k=1}^K\fronorm{\X_k(\I-\J)}^2
    \leq& \frac{1}{K\p}\sum_{j=0}^{\crounds-1}\sum_{i=1}^\cp (T_1+T_2) \\
    \leq& \lr^2\V\parenth{\frac{1+\spgap^2}{1-\spgap^2}\cp -1}  + \frac{2\lr^2\M\cp}{1-\spgap^2}\frac{1}{K}\sum_{k=1}^K \frac{\fronorm{\tg(\X_k)}^2}{\p}+ \nonumber\\
    &\frac{\lr^2\cp^2}{1-\spgap} \parenth{\frac{2\spgap^2}{1+\spgap}+\frac{2\spgap}{1-\spgap} +  \frac{\cp-1}{\cp}}\frac{1}{K}\sum_{k=1}^K\frac{\Exs\fronorm{\tg(\X_k)}^2}{\p}.
\end{align}
Substituting the expression of network error back to inequality \eqref{eqn:tight_thm1}, we obtain
\begin{align}
    \frac{1}{K}\sum_{k=1}^K\Exs\vecnorm{\tg(\genx_k)}^2
    \leq& \frac{2(\F(\x_1)-\F_{\text{inf}})}{\genlr K} +  \frac{\genlr\lip\V}{\p} + \lr^2\lip^2\V\parenth{\frac{1+\spgap^2}{1-\spgap^2}\cp -1}- \nonumber \\
        & \brackets{1-\genlr \lip\parenth{\frac{\M}{\p}+1} -\frac{2\lr^2\lip^2\M\cp}{1-\spgap^2}} \frac{1}{K}\sum_{k = 1}^{K} \frac{\Exs\fronorm{\tg(\X_k)}^2}{\p}+ \nonumber \\
        & \frac{\lr^2\lip^2\cp^2}{1-\spgap} \parenth{\frac{2\spgap^2}{1+\spgap}+\frac{2\spgap}{1-\spgap} + \frac{\cp-1}{\cp}}\frac{1}{K}\sum_{k=1}^K\frac{\Exs\fronorm{\tg(\X_k)}^2}{\p}.
\end{align}
When the learning rate satisfies
\begin{align}
    \genlr\lip \parenth{\frac{\M}{\p}+1} +\frac{2\lr^2\lip^2\M\cp}{1-\spgap^2}+ \frac{\lr^2\lip^2\cp^2}{1-\spgap} \parenth{\frac{2\spgap^2}{1+\spgap}+\frac{2\spgap}{1-\spgap} + \frac{\cp-1}{\cp}} \leq 1, \label{eqn:lr_cond}
\end{align}
we have
\begin{align}
    \frac{1}{K}\sum_{k=1}^K\Exs\vecnorm{\tg(\genx_k)}^2
    \leq& \frac{2(\F(\x_1)-\F_{\text{inf}})}{\genlr K} +  \frac{\genlr\lip\V}{\p} + \lr^2\lip^2\V\parenth{\frac{1+\spgap^2}{1-\spgap^2}\cp -1} \label{eqn:final}
\end{align}
where $\genlr = \p\lr /(\p+\addw)$ and $\spgap = \max\{|\lambda_2(\W)|, |\lambda_{\p+\addw}(\W)|\}$. Setting $\M = 0$, the condition on learning rate \eqref{eqn:lr_cond} can be further simplified as follows:
\begin{align}
    &\frac{\p}{\p+\addw}\lr\lip + \frac{\lr^2\lip^2\cp^2}{1-\spgap} \parenth{\frac{2\spgap^2}{1+\spgap}+\frac{2\spgap}{1-\spgap} + \frac{\cp-1}{\cp}} \\
    =& \frac{\p}{\p+\addw}\lr\lip + \frac{\lr^2\lip^2\cp^2}{(1-\spgap)^2} \parenth{\frac{2\spgap^2(1-\spgap)}{1+\spgap}+2\spgap + \frac{\cp-1}{\cp}(1-\spgap)} \\
    \leq& \frac{\p}{\p+\addw}\lr\lip + \frac{\lr^2\lip^2\cp^2}{(1-\spgap)^2}(2+2+1) \\
    =& \frac{\p}{\p+\addw}\lr\lip + \frac{5\lr^2\lip^2\cp^2}{(1-\spgap)^2} \leq 1. \label{eqn:lr_cond_simp}
\end{align}
Here, we complete the proof.

\section{Proof of Corollary 1 (Finite Horizon Result)}
Directly substituting $\lr = \frac{\p+\addw}{\lip\p}\sqrt{\frac{\p}{K}}$ into \eqref{eqn:final}, we have
\begin{align}
    \frac{1}{K}\sum_{k=1}^K\Exs\vecnorm{\tg(\genx_k)}^2
    \leq& \frac{2\lip(\F(\x_1)-\F_{\text{inf}})}{\sqrt{\p K}} +  \frac{\V}{\sqrt{\p K}} + \frac{\p}{K}\parenth{1+\frac{\addw}{\p}}^2\parenth{\frac{1+\spgap^2}{1-\spgap^2}\cp -1}\V. \label{eqn:rate1}
\end{align}
Note that the learning rate should satisfy the condition in \eqref{eqn:lr_cond_simp}. That is, the total iterations should satisfy:
\begin{align}
    \sqrt{\frac{\p}{K}} + \frac{5m}{K}\brackets{\parenth{1+\frac{\addw}{\p}}\frac{\cp}{1-\spgap}}^2 \leq 1.
\end{align}
When $K$ is sufficiently large, the first term can be arbitrarily small. In particular, when $K > 4\p$, the first term will be smaller than $1/2$. Then, it is enough to show the second term is smaller than $1/2$ as well.
\begin{align}
    &\frac{5m}{K}\brackets{\parenth{1+\frac{\addw}{\p}}\frac{\cp}{1-\spgap}}^2 \leq \frac{1}{2} \\
    \Rightarrow& K \geq 10\p\brackets{\parenth{1+\frac{\addw}{\p}}\frac{\cp}{1-\spgap}}^2. \label{eqn:lr_K}
\end{align}
Here, we complete the proof of the first part. Furthermore, when the communication period and total iterations satisfy
\begin{align}
    \frac{1}{\sqrt{\p K}} \geq \frac{(\p+\addw)}{K}\parenth{1+\frac{\addw}{\p}}\parenth{\frac{1+\spgap^2}{1-\spgap^2}\cp-1} \label{eqn:K_cond}
\end{align}
then the last term in \eqref{eqn:rate1} is smaller than the second term. As a result, we have
\begin{align}
     \frac{1}{K}\sum_{k=1}^K\Exs\vecnorm{\tg(\genx_k)}^2
    \leq \frac{2\lip(\F(\x_1)-\F_{\text{inf}})}{\sqrt{\p K}} +  \frac{2\V}{\sqrt{\p K}}. 
\end{align}
In order to get a lower bound on $K$ from \eqref{eqn:K_cond}, it is enough to show
\begin{align}
    &\frac{(\p+\addw)}{K}\parenth{1+\frac{\addw}{\p}}\parenth{\frac{1+\spgap^2}{1-\spgap^2}\cp-1} \\
    \leq& \frac{(\p+\addw)}{K}\parenth{1+\frac{\addw}{\p}}\frac{1+\spgap^2}{1+\spgap}\frac{\cp}{1-\spgap} \\
    \leq& \frac{(\p+\addw)}{K}\parenth{1+\frac{\addw}{\p}}\frac{\cp}{1-\spgap} \leq \frac{1}{\sqrt{\p K}} \\
    \Rightarrow& K \geq (\p+\addw)^2\p\brackets{\parenth{1+\frac{\addw}{\p}}\frac{\cp}{1-\spgap}}^2. \label{eqn:K_cond2}
\end{align}
Once $\p+\addw \geq \sqrt{10} \approx 3.1$, \eqref{eqn:K_cond2} is more strict than \eqref{eqn:lr_K}.

\section{Proof of Lemma 1 and Theorem 2: Best Choice of $\alpha$ in EASGD}
Recall that in EASGD, $\spgap = \max\{|1-\alpha|,|1-(\p+1)\alpha|\}$. It is straightforward to show that
\begin{align}
    \spgap = 
    \begin{cases}
    (\p+1)\alpha -1, & \frac{2}{\p+2} < \alpha \leq \frac{2}{\p+1} \\
    1-\alpha, & 0\leq \alpha \leq \frac{2}{\p+2}
    \end{cases}.
\end{align}
When $\alpha = \frac{2}{\p+2}$, one can get the minimal value of $\spgap$, which equals to $ 1-\alpha = (\p+1)\alpha -1 = \frac{\p}{\p+2}$. Then, substituting $\spgap = \frac{\p}{\p+2}, \cp = 1, \addw = 0$ into Theorem 1, we complete the proof of Theorem 2. 

\section{Proof of Lemma 2: Generalized Elastic Averaging}
Lemma 2 is built upon a known result about the eigenvalues of block matrices.
\begin{lem}[\cite{fiedler1974eigenvalues}]
\label{lem:eigen}
Let $\matA$ be a symmetric $m\times m$ matrix with eigenvalues $\lambda_1,\lambda_2,\dots,\lambda_m$, let $\mathbf{u}, \vecnorm{\mathbf{u}} =1$, be a unit eigenvector corresponding to $\lambda_1$; let $\matB$ be a symmetric $n \times n $ matrix with eigenvalues $\beta_1,\beta_2,\dots,\beta_n$, let $\mathbf{v}, \vecnorm{\mathbf{v}} =1$, be a unit eigenvector corresponding to $\beta_1$. Then for any $\rho$, the matrix
\begin{align}
    \mathbf{C} = 
    \begin{bmatrix}
    \matA & \rho\mathbf{u}\mathbf{v}\tp \\
    \rho\mathbf{v}\mathbf{u}\tp & \matB
    \end{bmatrix}
\end{align}
has eigenvalues $\lambda_2,\dots,\lambda_m,\beta_2,\beta_n,\gamma_1,\gamma_2$, where $\gamma_1,\gamma_2$ are eigenvalues of the matrix:
\begin{align}
    \hat{\mathbf{C}} = 
    \begin{bmatrix}
    \lambda_1 & \rho \\
    \rho & \beta_1
    \end{bmatrix}.
\end{align}
\end{lem}

In our case, recall the definition of $\W'$:
\begin{align}
    \W' = 
    \begin{bmatrix}
    (1-\alpha)\W & \alpha \one \\
    \alpha\one\tp & 1-\p\alpha
    \end{bmatrix}.
\end{align}
In order to apply \Cref{lem:eigen}, let us set $\matA = (1-\alpha)\W$. Accordingly, the eigenvalues of $\matA$ are $1-\alpha, (1-\alpha)\lambda_2,\dots,(1-\alpha)\lambda_\p$. The eigenvector corresponding to $1-\alpha$ is $\frac{\one}{\sqrt{\p}}$. Moreover, set $B = 1-\p\alpha$. Then, it has only one eigenvalue $1-\p\alpha$ and the corresponding eigenvector is scalar $1$. Substituting $\matA,B$ into $\W'$, we have
\begin{align}
    \W' = 
    \begin{bmatrix}
    \matA & \alpha\sqrt{\p}\cdot \frac{\one}{\sqrt{\p}}  \\
    \alpha\sqrt{\p}\cdot\frac{\one\tp}{\sqrt{\p}} & B
    \end{bmatrix}.
\end{align}
According to \Cref{lem:eigen}, the eigenvalues of $\W'$ are $(1-\alpha)\lambda_2,\dots,(1-\alpha)\lambda_\p,\gamma_1,\gamma_2$, where $\gamma_1,\gamma_2$ are eigenvalues of the matrix:
\begin{align}
    \hat{\mathbf{C}} = 
    \begin{bmatrix}
    1-\alpha & \alpha\sqrt{\p} \\
    \alpha\sqrt{\p} & 1-\p\alpha
    \end{bmatrix}.
\end{align}
For matrix $\hat{\mathbf{C}}$ we have
\begin{align}
    \gamma^2 - \brackets{2-(\p+1)\alpha}\gamma + 1-(\p+1)\alpha = 0
\end{align}
The above equation yields $\gamma_1 = 1, \gamma_2 = 1-(\p+1)\alpha$.

Finally, we have $\spgap' = \max\{|(1-\alpha)\lambda_2|, |(1-\alpha)\lambda_\p|,|1-(\p+1)\alpha|\} = \max\{(1-\alpha)\spgap, |1-(\p+1)\alpha|\}$. As a consequence, when $(1-\alpha)\spgap = (\p+1)\alpha-1$, i.e., $\alpha = \frac{1+\spgap}{\p+1+\spgap}$, the value of $\spgap'$ is minimized.

\end{document}